\documentclass{article}

\usepackage[preprint]{neurips_2026}

\usepackage[utf8]{inputenc} 
\usepackage[T1]{fontenc}    
\usepackage{hyperref}       
\usepackage{url}            
\usepackage{booktabs}       
\usepackage{amsfonts}       
\usepackage{nicefrac}       
\usepackage{microtype}      
\usepackage{xcolor}         

\usepackage{booktabs}
\usepackage{multirow}
\usepackage{makecell}
\usepackage[dvipsnames,table]{xcolor}
\usepackage{graphicx}
\usepackage{caption}
\usepackage{subcaption}
\usepackage{tikz}
\usepackage{wrapfig}
\usepackage{adjustbox}
\usepackage{array}
\usepackage{listings}
\usepackage{etoc}

\usepackage{arydshln}
\usepackage{algpseudocode}
\usepackage{xcolor} 
\usepackage[normalem]{ulem}
\definecolor{lightblue}{RGB}{230, 245, 255}
\definecolor{lightgreen}{RGB}{230, 255, 230}
\definecolor{lightpink}{RGB}{255, 230, 240}
\definecolor{headercolor}{RGB}{70, 130, 180}

\usepackage{amsmath}
\usepackage{amssymb}
\usepackage{mathtools}
\usepackage{amsthm}
\usepackage{enumitem}

\usepackage{algorithm}
\usepackage{algpseudocode}

\definecolor{cacheblue}{RGB}{28, 105, 164}
\definecolor{actionred}{RGB}{202, 0, 32}
\newtheorem{theorem}{Theorem}
\newtheorem{corollary}{Corollary}
\newtheorem{assumption}{Assumption}
\newtheorem{remark}{Remark}

\title{Flash-dLLM: IO-Aware KV Caching and Parallel Decoding for Fast, Memory-Efficient Diffusion LLMs}

\author{%
  Quan Nguyen-Tri \\
  VILA Lab, MBZUAI \\
  Abu Dhabi, UAE \\
  \texttt{quan.nguyen@mbzuai.ac.ae} \\
  \And 
  Mukul Ranjan \\ 
  VILA Lab, MBZUAI \\
  Abu Dhabi, UAE \\
  \texttt{mukul.ranjan@mbzuai.ac.ae} \\
  \And 
  Zhiqiang Shen\thanks{Corresponding author.} \\ 
  VILA Lab, MBZUAI \\
  Abu Dhabi, UAE \\
  \texttt{zhiqiang.shen@mbzuai.ac.ae} \\
  \AND
  \small Code available at: \url{https://github.com/VILA-Lab/Flash-dLLM} \\
}

\begin{document}

\maketitle

\begin{abstract}

    Diffusion Large Language Models (dLLMs) have recently emerged as a promising alternative to autoregressive LLMs by enabling non-autoregressive text generation. However, their practical deployment remains limited by inefficient inference, largely due to the absence of effective Key-Value (KV) caching and scalable parallel decoding mechanisms. Existing acceleration methods typically study KV caching and parallel decoding in isolation, overlooking the I/O bottlenecks that arise when cache reuse and parallel token verification are jointly applied. In this work, we introduce $\textbf{ Flash-dLLM}$, a training-free inference acceleration framework for fast and memory-efficient dLLMs. Flash-dLLM first identifies GPU memory I/O as a dominant bottleneck in KV-cache-enabled dLLM inference and addresses it with an I/O-aware fused KV-cache kernel that reduces redundant memory movement. Building on this optimized cache mechanism, Flash-dLLM further proposes an efficient KV-cache-driven draft-and-verify decoding strategy, where the dLLM itself serves as both drafter and verifier without requiring an auxiliary model. This unified design enables faster decoding while preserving generation quality and improving scalability to longer sequences and larger batch size. Extensive experiments on mathematical reasoning and code-generation benchmarks demonstrate that Flash-dLLM consistently outperforms existing state-of-the-art dLLM acceleration methods in both inference speed and memory efficiency. In particular, it achieves $5.1\times$ and $11.0\times$ speedups over prior strongest baseline Elastic-Cache on GSM8K and HumanEval, respectively.
\end{abstract}

\section{Introduction}

Diffusion Large Language Models (dLLMs)~\cite{austin2021structured,sahoo2024simple,li2025survey,bie2025llada2,nie2025large,ye2025dream,ou2024your,yang2025mmada,shi2024simplified,gemini_diffusion2025,mercury2025,arriola2025block,lou2023discrete,nie2025scaling,xie2025dream} have recently emerged as a promising alternative to conventional autoregressive large language models~\cite{radford2018improving,achiam2023gpt,singh2025openai,comanici2025gemini,guo2025deepseek,yang2025qwen3} by reformulating text generation as an iterative denoising process. Instead of producing tokens strictly from left to right, dLLMs refine partially masked or noisy sequences over multiple denoising steps, enabling more flexible generation schedules and exposing opportunities for non-autoregressive decoding. This paradigm has shown encouraging potential for controllable generation, global sequence refinement, and improved parallelism. However, despite these algorithmic advantages, the practical inference efficiency of open-source dLLMs still lags behind mature autoregressive LLM systems, whose deployment has benefited from years of optimization around KV caching, attention kernels, and speculative decoding~\cite{kwon2023efficient, ye2025flashinfer, cai2024medusa, pope2023efficiently}.

A major bottleneck to efficient dLLM inference lies in the repeated construction and movement of Key-Value (KV) states across denoising iterations. In autoregressive LLMs, KV caching avoids recomputing attention states for previously generated tokens. In dLLMs, however, the sequence is repeatedly revisited, and the set of updated tokens changes across iterations. This causes frequent KV-cache reads, writes, and updates~\cite{liu2025dllm,song2026sparse,ma2025dinfer}, even when many cached states remain unchanged or contribute little to the current decoding step. As a result, naive KV caching may reduce floating-point computation but introduce substantial GPU memory-access overhead. In practice, the redundancy in KV-cache read and write operations can dominate runtime, limiting the actual speedup obtainable from cache reuse (Fig.~\ref{fig:speed-compare}).

Beyond this system-level redundancy, dLLM decoding also exhibits strong token-level sparsity. Although each denoising step processes the full sequence, only a small fraction of decoded or partially decoded tokens has a significant influence on the current prediction distribution. Many tokens remain stable across iterations or provide limited additional context for the next decoding decision. Treating all cached tokens as equally important therefore wastes memory bandwidth and computation~\cite{song2026sparse,huang2026mask,zhang2023h2o,xiao2024efficient,tang2024quest}. This observation suggests that efficient dLLM inference should not simply cache and reuse all KV states uniformly, instead, it should prioritize the subset of tokens that meaningfully affect the current decoding process while reducing unnecessary memory traffic for less influential tokens (Fig.~\ref{fig:num-token}).

\begin{figure}[t]
  \centering
  \begin{subfigure}[b]{0.32\linewidth}
    \centering
    \includegraphics[width=1\linewidth]{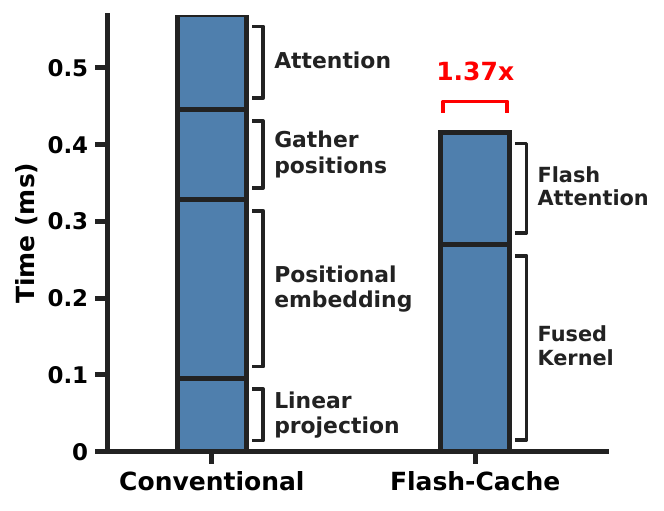}
    \vspace{-0.2in}
    \caption{IO bottleneck in KV caching}
    \label{fig:speed-compare}
  \end{subfigure}
  \hfill
  \begin{subfigure}[b]{0.33\linewidth}
    \centering
    \includegraphics[width=1\linewidth]{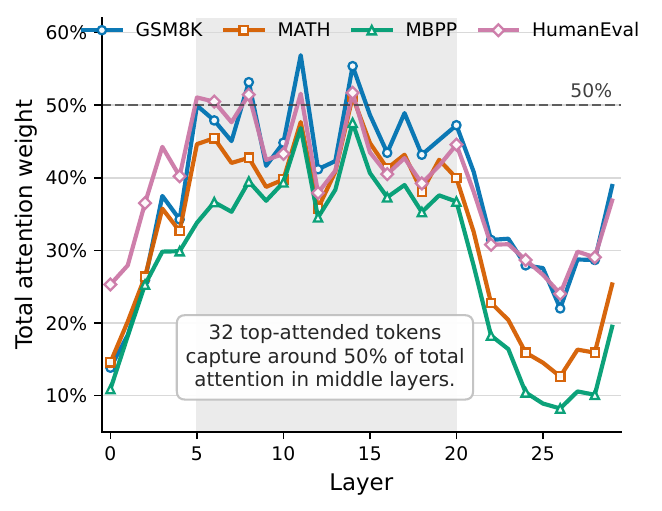}
    \vspace{-0.2in}
    \caption{Attention sparsity across layers}
    \label{fig:num-token}
  \end{subfigure}
  \hfill
  \begin{subfigure}[b]{0.33\linewidth}
    \centering
    \includegraphics[width=1\linewidth]{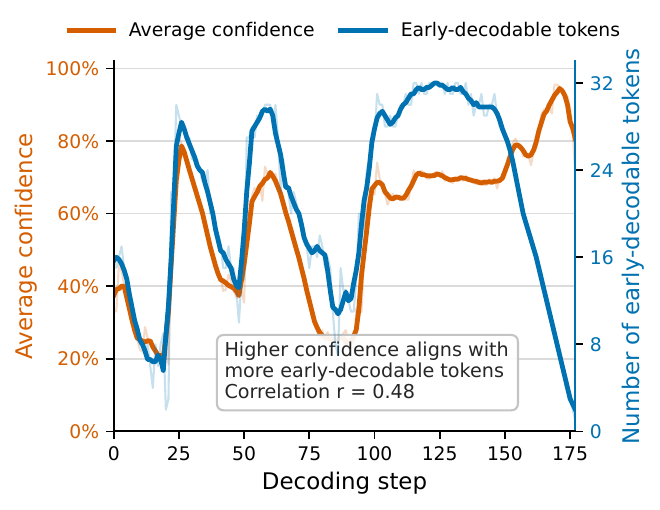}
    \vspace{-0.2in}
    \caption{Confidence and early decoding}
    \label{fig:average-prediction}
  \end{subfigure}
  \vspace{-0.08in}
  \caption{Motivating observations (LLaDA-1.5). (a)~Per-layer latency breakdown: the conventional cache pipeline is dominated by memory-bound operations; our fused kernel eliminates intermediate materialization for a 1.37$\times$ speedup. (b)~Top-32 most-attended decoded tokens capture up to 50\% of total attention in middle layers, motivating selective tracking over full recomputation. (c)~Average prediction confidence and number of early-decodable tokens are positively correlated across denoising steps, indicating that improving confidence directly increases tokens decoded per step.}
  \vspace{-0.15in}
\end{figure}

Another key point comes from the confidence dynamics of dLLM generation~\cite{wei2025accelerating,kong2025accelerating,israel2025accelerating,wudynamic}. During iterative denoising, many tokens become semantically determined at an early stage, even though their individual confidence scores may not yet exceed the conservative threshold required for immediate commitment. These tokens are often close to being correctly decoded, but existing decoding strategies delay their acceptance until later iterations, leading to redundant refinement steps. This creates a mismatch between token readiness and token commitment: the model already contains sufficient information to recover many positions, yet the decoding algorithm fails to exploit this early predictability due to insufficient confidence calibration.

Importantly, we observe that increasing the average confidence level of candidate decoded tokens directly improves the number of tokens that can be accepted correctly in each denoising step (Fig.~\ref{fig:average-prediction}). This indicates that the bottleneck is not only whether tokens can be predicted early, but whether their confidence can be made reliable enough for safe parallel commitment. A more effective decoding strategy should therefore aggregate useful contextual evidence, suppress low-impact cache interactions, and raise the confidence of promising token candidates. By doing so, dLLMs can decode more tokens per iteration while preserving generation quality, thereby reducing the total number of denoising steps required for completion.

Motivated by these observations, we propose {\bf Flash-dLLM}, a training-free inference framework that jointly optimizes IO-aware KV caching and cache-driven parallel decoding for fast, memory-efficient dLLMs. Flash-dLLM first reduces redundant KV-cache memory movement through an IO-aware fused cache kernel, which minimizes unnecessary read/write operations and improves cache locality during iterative denoising. It then exploits token-level sparsity by focusing cache reuse and verification on tokens that most affect the current decoding step. Finally, Flash-dLLM introduces a KV-cache-based parallel draft-and-verify mechanism that enables the dLLM itself to draft multiple early-decodable tokens and verify them with improved confidence, without requiring an auxiliary drafter or additional training. Together, these designs convert the inherent parallelism of diffusion language generation into practical wall-clock acceleration and improved memory scalability.

\begin{figure}
    \centering
    \includegraphics[width=1\linewidth]{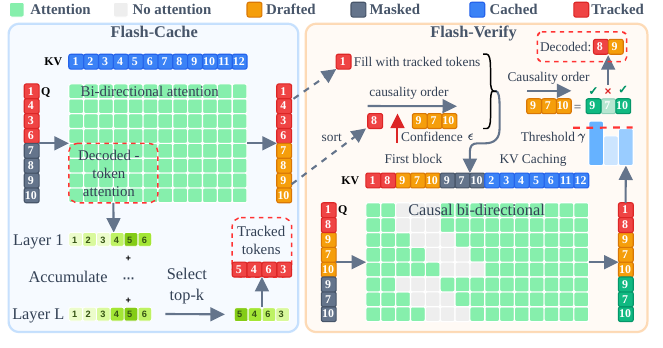}
    \caption{Overview of Flash-dLLM. \textbf{Left (Flash-Cache):} At each step, a fixed-size query of $\beta_m$ tokens (tracked + current masked window) attends bidirectionally to the full KV cache via the fused kernel. Tracked positions are the top-$k$ decoded tokens by attention importance, reselected each step to keep the cache-served representations aligned with full recomputation. \textbf{Right (Flash-Verify):} Draft predictions are sorted by confidence. Tokens above threshold $\epsilon$ are decoded directly. Remaining search tokens are duplicated in the query (once with the draft prediction, once with \texttt{[MASK]}) under a causal attention mask that prevents the two views from attending to each other. Tokens are accepted left-to-right when both views agree and the mask-view confidence exceeds $\gamma$.}
    \label{fig:overview}
\end{figure}

This study makes the following contributions:

\begin{itemize}
    \item We identify redundant read/write memory access as a key bottleneck in KV-cache-enabled dLLM inference, and propose an IO-aware fused KV-cache mechanism to reduce unnecessary GPU memory movement and improve cache locality. We further observe that only a small subset of decoded tokens significantly affects the current decoding step. Based on this, Flash-dLLM selectively prioritizes influential tokens during cache reuse and verification, improving efficiency without uniformly processing all cached states.
    \item We propose a training-free parallel decoding strategy where the dLLM itself serves as both drafter and verifier. By leveraging cached states to increase candidate-token confidence, Flash-dLLM enables more tokens to be decoded correctly in early stages, reducing denoising steps while preserving generation quality.
    \item Our approach achieves strong empirical speed and scalability gains. Extensive experiments on mathematical reasoning and code generation benchmarks demonstrate that the proposed Flash-dLLM improves inference speed and memory efficiency while maintaining generation quality, outperforming existing dLLM acceleration methods.
\end{itemize}

\section{Methodology}
\label{sec:method}

\subsection{Preliminaries}

\paragraph{KV Caching in Diffusion LLMs.} 
We adopt the sliding window decoding and KV caching strategy from Elastic-Cache \cite{nguyen2025attention}. Let $\mathcal{I} = \{1, \ldots, N\}$ represent all positions, $\mathcal{D}^{t}$ denote the newly decoded positions at step $t$, and $\mathcal{D}^{<t} = \mathcal{D}^{0} \cup \dots \cup \mathcal{D}^{t-1} $ denote the decoded positions up to step $t$, and $\mathcal{M}^{t}$ represent the remaining masked positions. At the initial step, we feed the entire sequence into the model to initialize the KV cache values for all positions $\mathcal{I}$: $\tilde{\mathbf{K}}^{t,l}_{[\mathcal{I}]}$ and $\tilde{\mathbf{V}}^{t,l}_{[\mathcal{I}]}$. For the subsequent step $t$, we perform KV caching for every token except the current query positions $\mathcal{Q}^{t} = \mathcal{D}^{t}\cup \mathcal{M}^{t}_{\beta_m}$. This includes the newly decoded tokens $\mathcal{D}^{t}$ and a fixed sliding window of masked tokens $\mathcal{M}^{t}_{\beta_m} = \mathcal{M}^{t}_{[1:\beta_\text{m}]}$, where $\beta_\text{m}$ is the fixed window size. The attention at step $t$ is computed as follows:

\begin{equation}\label{eq:kv-cache-dllm}
\mathbf{A}^{t,l}_{[\mathcal{Q}^t]} = \mathrm{softmax}\!\left(\frac{\mathbf{Q}^{t,l}_{[\mathcal{Q}^t]}(\tilde{\mathbf{K}}^{t,l}_{[\mathcal{I}]})^\top}{\sqrt{d_k}}\right)\tilde{\mathbf{V}}^{t,l}_{[\mathcal{I}]}, \quad \text{update: } \tilde{\mathbf{K}}^{t,l}_{[\mathcal{Q}^t]} = \mathbf{K}^{t,l}_{[\mathcal{Q}^t]},\;\; \tilde{\mathbf{V}}^{t,l}_{[\mathcal{Q}^t]} = \mathbf{V}^{t,l}_{[\mathcal{Q}^t]}.
\end{equation}

The cache is exact for positions in $\mathcal{Q}^t$ and approximate for the rest. Existing methods differ in how they manage this approximation: Fast-dLLM~\citep{wu2025fast} refreshes the entire cache at block boundaries, and Elastic-Cache~\citep{nguyen2025attention} triggers refresh adaptively based on attention-pattern drift. All these methods implement the cache update logic in PyTorch, issuing separate kernel launches for QKV projection, rotary positional embedding, cache writes, and attention computation per layer. 
This makes the caching operation memory-bound: the intermediate tensors are written to and read from GPU HBM multiple times, and the memory access cost dominates over the arithmetic cost.

\paragraph{Parallel Decoding in Diffusion LLMs.}
Diffusion LLMs generate by iteratively unmasking tokens from a fully masked sequence. At each step, the model predicts all masked positions in parallel, but these predictions are conditionally independent given $\mathbf{x}_t$. The model samples from the product of marginals $\prod_i p(\mathbf{x}_s^i \mid \mathbf{x}_t)$, while the true joint $p(\mathbf{x}_s^i, \mathbf{x}_s^j \mid \mathbf{x}_t) = p(\mathbf{x}_s^i \mid \mathbf{x}_t) \cdot p(\mathbf{x}_s^j \mid \mathbf{x}_t, \mathbf{x}_s^i)$ contains inter-token dependencies~\citep{wu2025fast}. Decoding many tokens at once amplifies this discrepancy and degrades coherence. Fast-dLLM~\citep{wu2025fast} tackles this issue by introducing \textit{Confidence-aware} decoding. This approach selectively unmasking only tokens whose confidence $c^i = \max_x p_\theta(x^i \mid \mathbf{x}_t)$ surpasses a predefined threshold $\epsilon$. Consequently, parallel decoding effectively approximates the true joint distribution when all decoded tokens exhibit high confidence.

\begin{figure}[htbp]
  \centering
  \begin{subfigure}[b]{0.48\linewidth}
    \centering
    \includegraphics[width=1\linewidth]{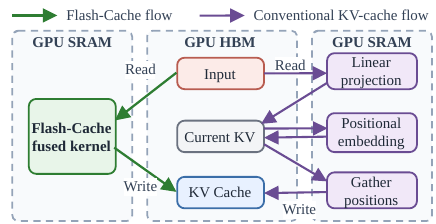}
    \vspace{-0.2in}
    \caption{Key-Value (KV) Caching in Flash-Cache}
    \label{fig:kv-cache-flash-attn}
  \end{subfigure}
  \hfill
  \begin{subfigure}[b]{0.48\linewidth}
    \centering
    \includegraphics[width=1\linewidth]{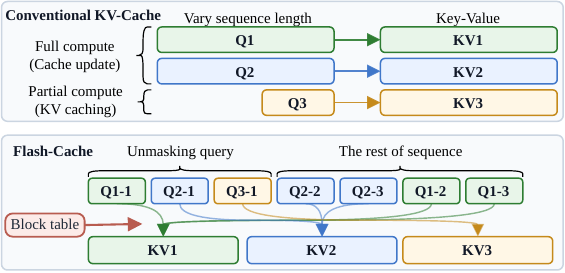}
    \vspace{-0.2in}
    \caption{Flash Attention in Flash-Cache}
    \label{fig:flash-attn-flash-cache}
  \end{subfigure}
  \vspace{-0.08in}
  \caption{IO-aware memory management in Flash-Cache. (a)~The fused kernel eliminates intermediate KV tensors by performing projection, RoPE, and cache write directly on the KV cache, reducing HBM read/write traffic and memory usage. (b)~Varying sequence lengths from partial (cached) and full (updated) computations are reorganized into contiguous blocks managed by a block table, enabling parallel execution without padding or synchronization overhead.}
  \vspace{-0.1in}
\end{figure}

\subsection{Flash-Cache: IO-aware key-value caching}
\label{subsec:flash-kv-cache}

Flash-Cache consists of three main components: (i) a fused kernel for Key-Value (KV) caching to alleviate the IO bottleneck, (ii) scheduled flash attention to manage the substantial variation in sequence lengths within batches, and (iii) constrained cache updates by monitoring the most-attended tokens.

\paragraph{KV caching fused kernel.}
At each transformer layer, a conventional KV cache implementation launches four separate CUDA kernels for the query set $\mathcal{Q}^t$: QKV projection, rotary positional embedding (RoPE), cache write, and attention. Each kernel writes its output to GPU high-bandwidth memory (HBM) before the next kernel reads it. QKV projection, RoPE, and cache writing each incur $\mathcal{O}(Qd_{\text{model}})$ memory traffic, while attention streams over the full cache with $\mathcal{O}(Nd)$ traffic. Thus, the total per-layer HBM traffic is approximately $4 \times \mathcal{O}(Qd_{\text{model}}) + \mathcal{O}(Nd)$. Because the non-attention operations have low arithmetic intensity, the cache-update path becomes memory-bound (Fig.~\ref{fig:kv-cache-flash-attn}).

Inspired by Flash Attention~\citep{dao2022flashattention}, we introduce a fused Flash-Cache kernel that combines QKV projection, RoPE, and cache writing (Fig.~\ref{fig:kv-cache-flash-attn}). The fused kernel performs projection and RoPE in SRAM and writes the resulting keys and values directly to the KV cache, eliminating intermediate key-value materialization. This reduces HBM traffic, lowers memory usage, and improves IO efficiency. As shown in Fig.~\ref{fig:speed-compare}, positional embedding and cache writing can dominate runtime relative to core computations such as attention and QKV projection. With Flash-Cache, we achieve a $1.37\times$ speedup on an RTX 3090 GPU.

\paragraph{Scheduled flash attention.}
KV caching and parallel decoding introduce new challenges for scaling diffusion LLMs to batched inference. While conventional attention can process variable-length sequences through padding and length-based grouping, diffusion LLMs make sequence lengths more dynamic and divergent. KV caching typically alternates between a caching stage, which computes only over a small window, and an update stage, which recomputes the full sequence; as a result, samples in the same batch may require substantially different computation at each iteration. Forcing all samples into the same stage can reduce efficiency and accuracy. This issue is further amplified by adaptive methods such as Elastic-Cache, where sequence length may vary within a layer, and by parallel decoding, which causes samples to progress at different rates.

To address this challenge, we extend Flash Attention~\citep{dao2022flashattention,shah2024flashattention,zadouri2026flashattention} by partitioning each batch into multiple sequence blocks and scheduling them with a block table that aligns query blocks with their corresponding key-value blocks (Fig.~\ref{fig:flash-attn-flash-cache}). This design mitigates length discrepancies during inference and provides a flexible mechanism for adding or removing blocks, enabling more adaptive decisions about when to cache or update. Unlike Flash Attention, which primarily optimizes IO efficiency within the attention computation, our Scheduled Flash Attention focuses on controlling which blocks are computed and in what order through the block table. This scheduling strategy is particularly effective for accelerating batched inference in diffusion LLMs.

The overall algorithm of Fused Kernel and Scheduled Flash Attention is presented in algorithm~\ref{alg:flash-cache}.

\begin{algorithm*}[t]
\caption{Flash-Cache}
\label{alg:flash-cache}
\scriptsize
\begin{algorithmic}[1]
\Require Input $\mathbf{X}\!\in\!\mathbb{R}^{M\times d}$, query $\mathbf{Q}\!\in\!\mathbb{R}^{M\times d}$, KV cache $\mathbf{K},\mathbf{V}\!\in\!\mathbb{R}^{(BN)\times d}$, output $\mathbf{A}\!\in\!\mathbb{R}^{M\times d}$, positions $\mathbf{P}\!\in\!\mathbb{R}^{M\times d}$, block size $\beta$, sequence length $N$, query length $M$, batch size $B$, heads $h$
\State Split $\mathbf{X}$ into $T_m=M/\beta$ blocks $\{\mathbf{X}_i\}_{i=1}^{T_m}$; create block table $\mathbf{B}$ and assign batch index $(\mathbf{X}_i,b)$
\State Split each block into heads $\mathbf{X}_{i,1},\dots,\mathbf{X}_{i,h}$
\end{algorithmic}

\vspace{0.3em}
\noindent
\begin{minipage}[t]{0.49\textwidth}
\textcolor{actionred}{\textbf{Flash Fused Kernel}}
\begin{algorithmic}[1]
\For{$1 \le i \le T_m$}
    \State Load $\mathbf{p}\gets\mathbf{P}_i$ 
    \hfill \textcolor{cacheblue}{\textit{// HBM $\to$ SRAM}}
    \For{$1 \le j \le h$}
        \State Load $\mathbf{x}\gets\mathbf{X}_{i,j}$ 
        \hfill \textcolor{cacheblue}{\textit{// HBM $\to$ SRAM}}
        \State $\mathbf{q},\mathbf{k},\mathbf{v}\gets\texttt{projection}(\mathbf{x})$; \quad $\mathbf{q},\mathbf{k}\gets\texttt{RoPE}(\mathbf{q},\mathbf{k})$
        \State Write $\mathbf{q}\to\mathbf{Q}_{i,j}$
        \State Write $\mathbf{k},\mathbf{v}\to\mathbf{K}_{[\mathbf{p},j]},\mathbf{V}_{[\mathbf{p},j]}$
        \hfill \textcolor{cacheblue}{\textit{// SRAM $\to$ HBM}}
    \EndFor
\EndFor
\end{algorithmic}
\end{minipage}
\hfill
\begin{minipage}[t]{0.49\textwidth}
\textcolor{actionred}{\textbf{Scheduled Flash Attention}}
\begin{algorithmic}[1]
\For{$1 \le i \le T_m$}
    \State Load $b\gets\mathbf{B}_i$
    \hfill \textcolor{cacheblue}{\textit{// HBM $\to$ SRAM}}
    \For{$1 \le j \le h$}
        \State Load $\mathbf{q}\gets\mathbf{Q}_{i,j}$
        \hfill \textcolor{cacheblue}{\textit{// HBM $\to$ SRAM}}
        \State Load $\mathbf{k},\mathbf{v}\gets\mathbf{K}_{[b,j]},\mathbf{V}_{[b,j]}$
        \hfill \textcolor{cacheblue}{\textit{// HBM $\to$ SRAM}}
        \State $\mathbf{a}\gets\texttt{attention}(\mathbf{q},\mathbf{k},\mathbf{v})$
        \State Write $\mathbf{a}\to\mathbf{A}_{i,j}$
        \hfill \textcolor{cacheblue}{\textit{// SRAM $\to$ HBM}}
    \EndFor
\EndFor
\end{algorithmic}
\end{minipage}

\vspace{0.2em}
\end{algorithm*}

\noindent \textbf{Selective cache update.}
Following our new design of KV caching and scheduled flash attention, we introduce a simple yet effective Constrained cache update to further enhance the scalability of our method. We observed that only 32 of the top-attended tokens can contribute approximately 50\% of the weight in attention computation among the middle layers (from layer 5 to 20, as shown in Fig.~\ref{fig:num-token}). This suggests that updating only a small subset of these tokens could be sufficient to retain most of the information loss. Motivated by this observation, we introduce the Selective cache update approach, which maintains and updates only a fixed set of the most-attended tokens (as depicted in Fig.~\ref{alg:flash-dllm}). 

Our method constructs a fixed-size query at each subsequent decoding step ($t>0$) using a sliding unmasking window of size $\beta_m$ and a fixed tracking budget $\beta_t$. Specifically, the query for step $t+1$ is defined as $\mathcal{Q}^{t+1}=\mathcal{M}_{\beta_m}^{t+1}\cup\mathcal{T}^{t+1},$
where $\mathcal{T}^{t+1}$ comprises the newly decoded tokens $\mathcal{D}^{t}$ and the most-attended tokens from the preceding step. The latter are selected from previously decoded tokens according to the attention they receive from masked tokens. The attention score of each decoded token $i$ at step $t$ is computed as follows:
\begin{equation}\label{eq:attn-importance}
a_i^t = \sum_{l=1}^{L} \frac{1}{H} \sum_{h=1}^{H} \frac{1}{|\mathcal{M}^t_{\beta_m}|}\sum_{j \in \mathcal{M}^t_{\beta_m}} \mathbf{S}^{t,l,h}_{j,i}, \quad i \in \mathcal{D}^{<t}
\end{equation}
where $\mathbf{S}^{t,l}$ is unnormalized attention logits. The score measures the extent to which the current masked queries are paying attention to the decoded position $i$. The tracking set is $\mathcal{T}^{t+1} = \operatorname{top\text{-}k}(\{a_i^t : i \in \mathcal{D}^{<t}\},\; \beta_\text{t} - |\mathcal{D}^{t}|) \cup \mathcal{D}^{t}$, and the next query set is $\mathcal{Q}^{t+1} = \mathcal{M}^{t+1}_{\beta_m} \cup \mathcal{T}^{t+1}$. Newly decoded tokens are prepended to $\mathcal{D}^{<t}$ before ranking, giving them automatic inclusion. All other decoded positions are served from cache without participating as queries, bounding per-step compute at $\beta_t+\beta_m$.

\begin{algorithm*}[t]
\caption{Flash-dLLM}
\begin{algorithmic}[1]
\Require Model $f_\theta$, prompt $\mathbf{x}_{\text{prompt}}$, generation length $N$, tracking budget $\beta_t$, masked-window size $\beta_m$, confidence threshold $\epsilon$, verify threshold $\gamma$.
\State $\mathbf{x}^0 \!\leftarrow\! \{\mathbf{x}_{\text{prompt}};\, \texttt{[MASK]}^N\}$; $\mathcal{D}^{<1} \!\leftarrow\! \{1,\dots,|\mathbf{x}_{\text{prompt}}|\}$; $\mathcal{M}^1 \!\leftarrow\! \{|\mathbf{x}_{\text{prompt}}|{+}1,\dots,|\mathbf{x}_{\text{prompt}}|{+}N\}$;

\State Allocate KV cache $\tilde{\mathbf{K}}^l, \tilde{\mathbf{V}}^l \in \mathbb{R}^{(B \cdot N_{\max}) \times d_{\text{model}}}$ for each layer $l$ \hfill \textcolor{cacheblue}{\textit{// pre-allocate once}}

\While{$\mathcal{M}^t \neq \emptyset$}
    \State $\mathcal{M}^t_{\beta_m} \leftarrow \mathcal{M}^t_{[:\beta_m]}$;\quad $\mathcal{Q}^t \leftarrow \mathcal{T}^{t} \cup \mathcal{M}^{t}_{\beta_m}$

    \State $p_\theta(\mathbf{x} \mid \mathbf{x}^t),\, \mathbf{S} \leftarrow \texttt{FusedForward}(f_\theta,\, \mathbf{x}^t_{[\mathcal{Q}^t]},\, \tilde{\mathbf{K}},\, \tilde{\mathbf{V}})$ \hfill \textcolor{actionred}{\textit{// Flash-Cache, Alg.~\ref{alg:flash-cache}}}

    \State $c^i, \hat{x}^i  \leftarrow \max_x\, p_\theta(x^i \mid \mathbf{x}^t)$ \quad for $i \in \mathcal{M}^t_{\beta_m}$;\quad $\mathcal{D}^t \leftarrow \{i \in \mathcal{M}^t_{\beta_m}: c^i \geq \epsilon\}$ ;\quad $\mathcal{S}^t \leftarrow \mathcal{M}^t_{\beta_m} \setminus \mathcal{D}^t$ 

    \If{$\texttt{verify}$ \textbf{and} $\mathcal{S}^t \neq \emptyset$} \hfill \textcolor{actionred}{\textit{// Flash-Verify}}
        \State Adjust tracking set $\mathcal{T}^{t}_v \gets \mathcal{T}^{t}_{[:2\beta_m - |\mathcal{D}^t| - 2|\mathcal{S}^t|]}$
        \State Construct verify query: $\{ \mathbf{x}_{[\mathcal{T}^{t}_v]}^t; \mathbf{\hat{x}}_{[\mathcal{D}^t]};\;\mathbf{\hat{x}}_{[\mathcal{S}^t]}; \mathbf{x}^t_{[\mathcal{S}^t]}\}$ 
        \State Construct causal mask for bi-directional attention $\mathbf{M}$
        \State $p_\theta(\mathbf{x} \mid \mathbf{\hat{x}}_{[\mathcal{S}^t]}) \leftarrow \texttt{FusedForward}(f_\theta,\, \text{verify query},\, \tilde{\mathbf{K}},\, \tilde{\mathbf{V}},\, \mathbf{M})$
        \State $\tilde{c}^i, \tilde{x}^i  \leftarrow \max_x\, p_\theta(x^i \mid \mathbf{\hat{x}}_{[\mathcal{S}^t]})$ \quad for $i \in \mathcal{S}^t$
        \State $\mathcal{D}^t \leftarrow \mathcal{D}^t \cup \{i \in \mathcal{S}^t : \hat{x}^i = \tilde{x}^i \;\wedge\; \tilde{c}^i \geq \gamma,\; \text{up to first mismatch}\}$
    \EndIf

    \State Decode: $\mathbf{x}^{t+1}_{[\mathcal{D}^t]} \leftarrow \mathbf{\hat{x}}_{[\mathcal{D}^t]}$;\quad $\mathcal{D}^{<t+1} \leftarrow \mathcal{D}^{<t} \cup \mathcal{D}^t$;\quad $\mathcal{M}^{t+1} \leftarrow \mathcal{M}^t \setminus \mathcal{D}^t$

    \State $a_i^t \leftarrow \sum_l \frac{1}{H|\mathcal{M}^t_{\beta_m}|} \sum_{h,\,j \in \mathcal{M}^t_{\beta_m}} \mathbf{S}^{t,l,h}_{j,i}$ \quad for $i \in \mathcal{D}^{<t}$ 
    \State $\mathcal{T}^{t+1} = \operatorname{top\text{-}k}(\{a_i^t : i \in \mathcal{D}^{<t}\},\; \beta_\text{t} - |\mathcal{D}^{t}|) \cup \mathcal{D}^{t}$ \quad $t \leftarrow t + 1$
\EndWhile
\State \Return $\mathbf{x}^t$
\end{algorithmic}
\label{alg:flash-dllm}
\end{algorithm*}

\vspace{-0.1in}
\subsection{Flash-Verify: KV-cache-driven draft-and-verify Parallel Decoding}
\label{subsec:flash-draft-verify}

Confidence-aware decoding~\citep{wu2025fast,wu2026fast} unmasks only tokens whose confidence $c^i = \max_x p_\theta(x^i \mid \mathbf{x}_t)$ exceeds a threshold $\epsilon$, discarding all others even when many are correct (Fig.~\ref{fig:average-prediction}). This creates a throughput ceiling on tasks where the model is uncertain, as few tokens pass the threshold per step. We propose Flash-Verify, a self-verification scheme in which the dLLM serves as both drafter and verifier, recovering correct predictions that confidence-aware decoding would waste. At each denoising step, a standard \textbf{draft pass} runs the model on $\mathcal{Q}^t = \mathcal{T}^t \cup \mathcal{M}^t_{\beta_m}$ against the full KV cache. The masked positions are sorted by confidence and partitioned into a confident set $\mathcal{D}^t$ (above $\epsilon$, accepted directly) and a search set $\mathcal{S}^t$ (below $\epsilon$, candidates for verification). 

A \textbf{verify pass} then constructs a new query with three groups: an adjusted tracking set $\mathcal{T}_v$ consisting of previously decoded tokens and $\mathcal{D}^t$, the search positions $\mathcal{S}^t$ filled with their draft predictions $\hat{x}^i$ (the \emph{draft view}), and the same positions filled with \texttt{[MASK]} (the \emph{mask view}). Both views share positional embeddings but are isolated by a causal attention mask loaded inside the fused Triton kernel: the tracked context cannot attend to the draft view, and the draft and mask views at the same position cannot attend to each other, so they produce independent predictions from shared context. A search token is accepted if both views agree and the mask-view confidence exceeds a threshold $\gamma$: $\text{accept}(i) = \mathbb{I}[\hat{x}^i = \tilde{x}^i] \cdot \mathbb{I}[\tilde{c}^i \geq \gamma]$,

where $\tilde{x}^i$ and $\tilde{c}^i$ are the mask view's prediction and confidence. Following speculative decoding conventions~\citep{leviathan2023fast}, tokens are accepted sequentially along the decoding order and all tokens after the first mismatch are rejected. 

The verify pass reuses the same fused kernel and pre-allocated KV cache; only the query tokens and the attention mask change, so the additional cost is proportional to $2\beta_m$ rather than the full sequence. Unlike prior draft-and-verify methods for dLLMs that rely on a separate autoregressive verifier~\citep{hu2025flashdlm} or multiple independent forward passes~\citep{wu2025free}, Flash-Verify requires no external model: the dLLM verifies its own predictions through the two-view attention mask, roughly doubling the tokens accepted per step (Fig.~\ref{fig:token per step}). Algorithm~\ref{alg:flash-dllm} summarizes the full procedure of Flash-dLLM, including both Flash-Cache and Flash-Verify.

\section{Experiments}
\label{sec:exp}
\subsection{Experimental Setup}
\label{subsec:exp-setup}

\noindent{\bf Implementation Details.} All experiments run on a single NVIDIA A100 80GB GPU. We evaluate \textbf{Flash-dLLM} on LLaDA-1.5~\citep{zhu2025llada} across GSM8K~\citep{cobbe2021training}, MATH~\citep{hendrycks2021measuring}, HumanEval~\citep{chen2021evaluating}, and MBPP~\citep{austin2021program}. We implement the fused KV-cache kernel in Triton 2.0. Default benchmark is GSM8K, with default hyperparameters: confidence threshold $\epsilon=0.9$, verify threshold $\gamma=0.8$, block size $\beta=16$, tracked budget $\beta_t=80$, sliding window size $\beta_m=64$, generation length 512. For fair comparison, we re-run all baselines under identical hardware and software configurations. \noindent{\bf Baselines.} We compare against three approaches: (1)~\textbf{No Cache}: standard dLLM inference without KV caching, under both greedy (fixed-step) and confidence-aware decoding; (2)~\textbf{Fast-dLLM}~\citep{wu2025fast}: prefix-caching with confidence-aware decoding; (3)~\textbf{Elastic-Cache}~\citep{nguyen2025attention}: adaptive KV caching with attention-pattern-based cache reuse. We report Flash-dLLM results under three configurations: greedy decoding (pure KV-cache speedup), confidence-aware decoding (KV-cache + parallel decoding), and Flash-Verify (KV-cache + draft-and-verify parallel decoding). 

\begin{figure}[t]
  \centering
  \begin{subfigure}[b]{0.48\linewidth}
    \centering
    \includegraphics[width=1\linewidth]{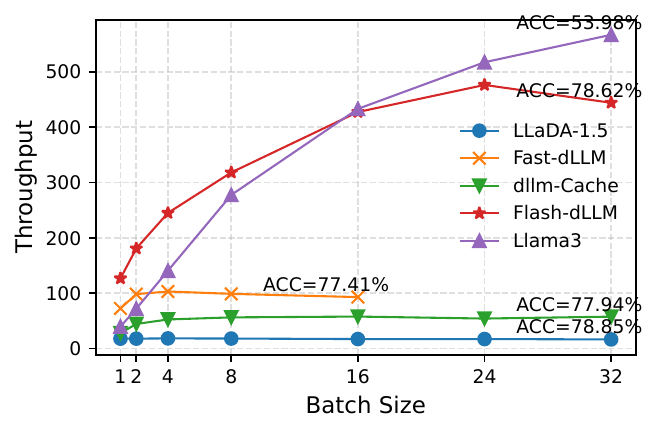}
    \vspace{-0.2in}
    \caption{Throughput vs.\ batch size}
    \label{fig:window-vs-block}
  \end{subfigure}
  \hfill
  \begin{subfigure}[b]{0.48\linewidth}
    \centering
    \includegraphics[width=1\linewidth]{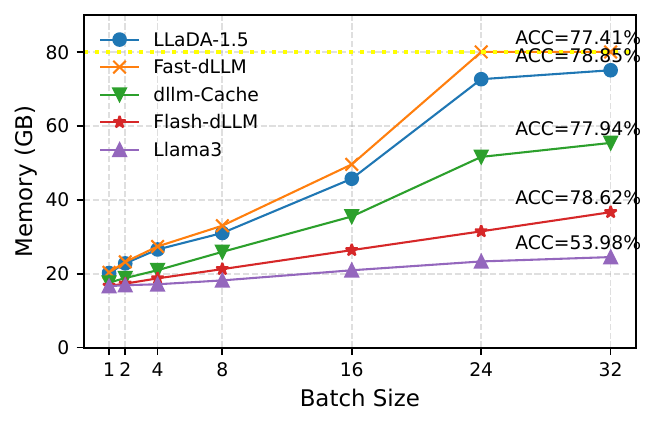}
    \vspace{-0.2in}
    \caption{Peak memory vs.\ batch size}
    \label{fig:update-freq}
  \end{subfigure}
  \vspace{-0.08in}
  \caption{Scalability on GSM8K-512 (1-shot, LLaDA-1.5). (a)~Flash-dLLM (Flash-Cache+Flash-Verify) throughput scales linearly to batch size 32, while Fast-dLLM OOMs at 24. Llama3-8B is an autoregressive reference. (b)~Flash-dLLM uses less memory than all baselines across batch sizes.}
  \label{fig:scalability}
  \vspace{-0.1in}
\end{figure}
\subsection{Main Results}

\begin{table}[t]
\centering
\caption{Accuracy and decoding efficiency of LLaDA-1.5 across different benchmarks and decoding configurations. Each cell reports accuracy (top) and throughput with speedup over greedy decoding without caching (bottom; \textcolor{blue}{blue}: tokens/s, \textcolor{orange}{orange}: speedup). \textbf{Bold} indicates the highest accuracy in each row, while \colorbox{yellow!20}{yellow shading} indicates the highest throughput.}
\label{tab:lla-1.5-instruct}
\resizebox{0.99\linewidth}{!}{%
\begin{tabular}{cc|cc|cccc|cc}
\toprule
\multicolumn{2}{c|}{} &
\multicolumn{2}{c|}{\bfseries Greedy} &
\multicolumn{4}{c|}{\bfseries Confidence-Aware} &
\multicolumn{2}{c}{\bfseries Flash-Verify} \\

\multicolumn{1}{c}{\bfseries Benchmark} &
\multicolumn{1}{c|}{\bfseries Len} &
\multicolumn{1}{c}{\bfseries No Cache} &
\multicolumn{1}{c|}{\bfseries Flash-Cache} &
\multicolumn{1}{c}{\bfseries No Cache} &
\multicolumn{1}{c}{\bfseries Fast-dLLM} &
\multicolumn{1}{c}{\bfseries Elastic-Cache} &
\multicolumn{1}{c|}{\bfseries Flash-Cache} &
\multicolumn{1}{c}{\bfseries No Cache} &
\multicolumn{1}{c}{\bfseries Flash-Cache} \\

\midrule
\multirow{2}{*}{\makecell{GSM8K \\ (5-shot)}}
& 256
& \makecell{80.36 \\ {\footnotesize
  \textcolor{blue}{6.7} (\textcolor{orange}{1.0$\times$})}}
& \makecell{82.87 \\ {\footnotesize
  \textcolor{blue}{56.8} (\textcolor{orange}{8.5$\times$})}}
& \makecell{80.44 \\ {\footnotesize
  \textcolor{blue}{22.5} (\textcolor{orange}{3.4$\times$})}}
& \makecell{80.59 \\ {\footnotesize
  \textcolor{blue}{51.2} (\textcolor{orange}{7.6$\times$})}}
& \makecell{81.88 \\ {\footnotesize
  \textcolor{blue}{45.9} (\textcolor{orange}{6.9$\times$})}}
& \makecell{82.34 \\ {\footnotesize
  \textcolor{blue}{144.9} (\textcolor{orange}{21.6$\times$})}}
& \makecell{\textbf{83.62} \\ {\footnotesize
  \textcolor{blue}{38.2} (\textcolor{orange}{5.7$\times$})}}
& \makecell{81.88 \\ \cellcolor{yellow!20}{\footnotesize
  \textcolor{blue}{194.9} (\textcolor{orange}{29.1$\times$})}} \\

& 512
& \makecell{81.35 \\ {\footnotesize
  \textcolor{blue}{2.6} (\textcolor{orange}{1.0$\times$})}}
& \makecell{82.94 \\ {\footnotesize
  \textcolor{blue}{54.8} (\textcolor{orange}{21.1$\times$})}}
& \makecell{81.88 \\ {\footnotesize
  \textcolor{blue}{17.2} (\textcolor{orange}{6.6$\times$})}}
& \makecell{80.82 \\ {\footnotesize
  \textcolor{blue}{36.8} (\textcolor{orange}{14.2$\times$})}}
& \makecell{82.79 \\ {\footnotesize
  \textcolor{blue}{41.7} (\textcolor{orange}{16.0$\times$})}}
& \makecell{82.87 \\ {\footnotesize
  \textcolor{blue}{149.4} (\textcolor{orange}{57.5$\times$})}}
& \makecell{82.71 \\ {\footnotesize
  \textcolor{blue}{32.2} (\textcolor{orange}{12.4$\times$})}}
& \makecell{\textbf{83.02} \\ \cellcolor{yellow!20}{\footnotesize
  \textcolor{blue}{210.6} (\textcolor{orange}{81.0$\times$})}} \\

\midrule
\multirow{2}{*}{\makecell{MATH \\ (4-shot)}}
& 256
& \makecell{33.52 \\ {\footnotesize
  \textcolor{blue}{8.5} (\textcolor{orange}{1.0$\times$})}}
& \makecell{\textbf{37.22} \\ {\footnotesize
  \textcolor{blue}{67.6} (\textcolor{orange}{8.0$\times$})}}
& \makecell{33.60 \\ {\footnotesize
  \textcolor{blue}{22.3} (\textcolor{orange}{2.6$\times$})}}
& \makecell{32.74 \\ {\footnotesize
  \textcolor{blue}{44.4} (\textcolor{orange}{5.2$\times$})}}
& \makecell{33.26 \\ {\footnotesize
  \textcolor{blue}{40.6} (\textcolor{orange}{4.8$\times$})}}
& \makecell{36.80 \\ {\footnotesize
  \textcolor{blue}{144.3} (\textcolor{orange}{17.0$\times$})}}
& \makecell{36.98 \\ {\footnotesize
  \textcolor{blue}{38.1} (\textcolor{orange}{4.5$\times$})}}
& \makecell{36.56 \\ \cellcolor{yellow!20}{\footnotesize
  \textcolor{blue}{189.7} (\textcolor{orange}{22.3$\times$})}} \\

& 512
& \makecell{35.63 \\ {\footnotesize
  \textcolor{blue}{5.0} (\textcolor{orange}{1.0$\times$})}}
& \makecell{37.40 \\ {\footnotesize
  \textcolor{blue}{66.2} (\textcolor{orange}{13.2$\times$})}}
& \makecell{35.56 \\ {\footnotesize
  \textcolor{blue}{20.3} (\textcolor{orange}{4.1$\times$})}}
& \makecell{33.68 \\ {\footnotesize
  \textcolor{blue}{44.4} (\textcolor{orange}{8.9$\times$})}}
& \makecell{35.84 \\ {\footnotesize
  \textcolor{blue}{41.4} (\textcolor{orange}{8.3$\times$})}}
& \makecell{37.08 \\ {\footnotesize
  \textcolor{blue}{149.9} (\textcolor{orange}{30.0$\times$})}}
& \makecell{\textbf{37.76} \\ {\footnotesize
  \textcolor{blue}{32.8} (\textcolor{orange}{6.6$\times$})}}
& \makecell{35.98 \\ \cellcolor{yellow!20}{\footnotesize
  \textcolor{blue}{210.1} (\textcolor{orange}{42.0$\times$})}} \\

\midrule
\multirow{2}{*}{\makecell{HumanEval \\ (0-shot)}}
& 256
& \makecell{\textbf{43.29} \\ {\footnotesize
  \textcolor{blue}{7.0} (\textcolor{orange}{1.0$\times$})}}
& \makecell{42.68 \\ {\footnotesize
  \textcolor{blue}{79.0} (\textcolor{orange}{11.3$\times$})}}
& \makecell{42.68 \\ {\footnotesize
  \textcolor{blue}{17.5} (\textcolor{orange}{2.5$\times$})}}
& \makecell{34.75 \\ {\footnotesize
  \textcolor{blue}{18.7} (\textcolor{orange}{2.7$\times$})}}
& \makecell{36.59 \\ {\footnotesize
  \textcolor{blue}{20.9} (\textcolor{orange}{3.0$\times$})}}
& \makecell{40.85 \\ {\footnotesize
  \textcolor{blue}{169.4} (\textcolor{orange}{24.2$\times$})}}
& \makecell{37.20 \\ {\footnotesize
  \textcolor{blue}{63.2} (\textcolor{orange}{9.0$\times$})}}
& \makecell{39.63 \\ \cellcolor{yellow!20}{\footnotesize
  \textcolor{blue}{209.2} (\textcolor{orange}{29.9$\times$})}} \\

& 512
& \makecell{40.85 \\ {\footnotesize
  \textcolor{blue}{3.2} (\textcolor{orange}{1.0$\times$})}}
& \makecell{41.46 \\ {\footnotesize
  \textcolor{blue}{74.5} (\textcolor{orange}{23.3$\times$})}}
& \makecell{39.63 \\ {\footnotesize
  \textcolor{blue}{9.7} (\textcolor{orange}{3.0$\times$})}}
& \makecell{36.59 \\ {\footnotesize
  \textcolor{blue}{15.4} (\textcolor{orange}{4.8$\times$})}}
& \makecell{37.80 \\ {\footnotesize
  \textcolor{blue}{16.8} (\textcolor{orange}{5.2$\times$})}}
& \makecell{\textbf{42.07} \\ {\footnotesize
  \textcolor{blue}{145.7} (\textcolor{orange}{45.5$\times$})}}
& \makecell{37.20 \\ {\footnotesize
  \textcolor{blue}{57.5} (\textcolor{orange}{18.0$\times$})}}
& \makecell{40.24 \\ \cellcolor{yellow!20}{\footnotesize
  \textcolor{blue}{185.6} (\textcolor{orange}{58.0$\times$})}} \\

\midrule
\multirow{2}{*}{\makecell{MBPP \\ (3-shot)}}
& 256
& \makecell{38.00 \\ {\footnotesize
  \textcolor{blue}{2.4} (\textcolor{orange}{1.0$\times$})}}
& \makecell{41.20 \\ {\footnotesize
  \textcolor{blue}{62.0} (\textcolor{orange}{25.8$\times$})}}
& \makecell{38.00 \\ {\footnotesize
  \textcolor{blue}{14.2} (\textcolor{orange}{5.9$\times$})}}
& \makecell{34.60 \\ {\footnotesize
  \textcolor{blue}{28.0} (\textcolor{orange}{11.7$\times$})}}
& \makecell{41.20 \\ {\footnotesize
  \textcolor{blue}{32.7} (\textcolor{orange}{13.6$\times$})}}
& \makecell{\textbf{41.80} \\ {\footnotesize
  \textcolor{blue}{115.9} (\textcolor{orange}{48.3$\times$})}}
& \makecell{41.40 \\ {\footnotesize
  \textcolor{blue}{37.7} (\textcolor{orange}{15.7$\times$})}}
& \makecell{38.20 \\ \cellcolor{yellow!20}{\footnotesize
  \textcolor{blue}{148.0} (\textcolor{orange}{61.7$\times$})}} \\

& 512
& \makecell{38.20 \\ {\footnotesize
  \textcolor{blue}{1.0} (\textcolor{orange}{1.0$\times$})}}
& \makecell{39.80 \\ {\footnotesize
  \textcolor{blue}{58.5} (\textcolor{orange}{58.5$\times$})}}
& \makecell{38.60 \\ {\footnotesize
  \textcolor{blue}{11.5} (\textcolor{orange}{11.5$\times$})}}
& \makecell{36.20 \\ {\footnotesize
  \textcolor{blue}{17.8} (\textcolor{orange}{17.8$\times$})}}
& \makecell{39.00 \\ {\footnotesize
  \textcolor{blue}{32.8} (\textcolor{orange}{32.8$\times$})}}
& \makecell{\textbf{40.20} \\ {\footnotesize
  \textcolor{blue}{102.2} (\textcolor{orange}{102.2$\times$})}}
& \makecell{39.40 \\ {\footnotesize
  \textcolor{blue}{31.6} (\textcolor{orange}{31.6$\times$})}}
& \makecell{39.00 \\ \cellcolor{yellow!20}{\footnotesize
  \textcolor{blue}{148.2} (\textcolor{orange}{148.2$\times$})}} \\

\bottomrule
\end{tabular}%
}
\end{table}

Table~\ref{tab:lla-1.5-instruct} compares the accuracy and decoding efficiency of the evaluated KV-caching and parallel decoding on mathematical reasoning and code-generation benchmarks.

\noindent{\bf Throughput.}
Flash-Cache substantially accelerates greedy decoding, yielding speedups of $8.0\times$--$58.5\times$. Under confidence-aware decoding, the range increases to $17.0\times$--$102.2\times$, indicating that cache acceleration remains effective with parallel decoding. Combining Flash-Verify and Flash-Cache achieves the highest throughput in all eight settings, reaching $148.0$--$210.6$ tokens/s and speedups of $22.3\times$--$148.2\times$. Relative to confidence-aware Flash-Cache, the second-fastest configuration throughout, it improves throughput by approximately $23.5\%$--$45.0\%$. The gains are larger at longer generation lengths: from 256 to 512 tokens, the speedup increases from $29.1\times$ to $81.0\times$ on GSM8K, $22.3\times$ to $42.0\times$ on MATH, $29.9\times$ to $58.0\times$ on HumanEval, and $61.7\times$ to $148.2\times$ on MBPP.

\noindent{\bf Accuracy.}
Accuracy exhibits a task-dependent trade-off. On GSM8K-512, Flash-Verify with Flash-Cache achieves both the highest accuracy ($83.02\%$) and throughput ($210.6$ tokens/s). Elsewhere, the fastest configuration is not consistently the most accurate: Flash-Verify without caching performs best on GSM8K-256 and MATH-512, whereas confidence-aware Flash-Cache leads on HumanEval-512 and both MBPP settings. On mathematical reasoning tasks, the combined method remains within $1.78$ percentage points of the best accuracy. Larger gaps arise for 256-token code generation, reaching $3.66$ points on HumanEval and $3.60$ points on MBPP. Overall, Flash-Verify with Flash-Cache provides the strongest throughput-oriented configuration.

\begin{table*}[t]
\centering
\caption{Comparison of accuracy and decoding throughput across different methods.}
\label{tab:accuracy-throughput}
\resizebox{\textwidth}{!}{%
\begin{tabular}{l|ccccccccc}
\toprule
\textbf{Metric}
& \textbf{dKV-Cache}
& \textbf{FlashDLM}
& \textbf{dLLM-Cache}
& \textbf{Dyna-dLLM}
& \textbf{Fast-dLLM}
& \textbf{FreeDave}
& \textbf{Elastic-Cache}
& \shortstack{\textbf{Flash-Cache}\\\textbf{(conf-aware)}}
& \shortstack{\textbf{Flash-Cache}\\\textbf{+ Flash-Verify}} \\
\midrule
\textbf{Acc. (\%)}
& 81.50
& 79.91
& 80.97
& 79.32
& 80.82
& 80.97
& 82.79
& 82.87
& \textbf{83.02} \\
\textbf{TPS}
& \textcolor{blue}{14.9} (\textcolor{orange}{5.7$\times$})
& \textcolor{blue}{15.7} (\textcolor{orange}{6.0$\times$})
& \textcolor{blue}{16.8} (\textcolor{orange}{6.5$\times$})
& \textcolor{blue}{38.4} (\textcolor{orange}{14.8$\times$})
& \textcolor{blue}{36.8} (\textcolor{orange}{14.2$\times$})
& \textcolor{blue}{42.8} (\textcolor{orange}{16.5$\times$})
& \textcolor{blue}{41.7} (\textcolor{orange}{16.0$\times$})
& \textcolor{blue}{149.4} (\textcolor{orange}{57.5$\times$})
& \textcolor{blue}{210.6} (\textcolor{orange}{81.0$\times$}) \\
\bottomrule
\end{tabular}%
}
\end{table*}

\noindent{\bf More baselines.}
Table~\ref{tab:accuracy-throughput} compares the accuracy and decoding throughput of the evaluated cache and parallel-verification methods. Conventional baselines achieve accuracies of $79.32\%$--$82.79\%$ and throughputs of $14.9$--$42.8$ tokens/s. Elastic-Cache yields the highest baseline accuracy ($82.79\%$), whereas FreeDave achieves the highest baseline throughput ($42.8$ tokens/s). 
Flash-Cache with confidence-aware decoding achieves $82.87\%$ accuracy and $149.4$ tokens/s, while its integration with Flash-Verify further improves performance to $83.02\%$ accuracy and $210.6$ tokens/s. These results demonstrate that combining Flash-Cache with Flash-Verify yields substantial gains in both KV-cache efficiency and parallel-decoding performance over existing baselines.

\subsection{Ablation Studies and Analysis}

\noindent\textbf{Scalability.}
Figure~\ref{fig:window-vs-block} reports throughput as a function of batch size on GSM8K-512 (1-shot) using LLaDA-1.5. Flash-dLLM scales nearly linearly to a batch size of $32$ without exhausting GPU memory, whereas Fast-dLLM encounters an out-of-memory error at batch size $24$. Table~\ref{tab:ablation-bs-speed} extends this analysis to GSM8K-512 (5-shot), showing that all Flash-dLLM configurations (greedy, confident, verify) scale monotonically from batch sizes $1$ to $32$. In particular, Flash-Cache with Flash-Verify reaches $199.8$ tokens/s at batch size $32$. Figure~\ref{fig:update-freq} further shows that, at batch size $16$, Flash-dLLM uses approximately $26$ GB of GPU memory, compared with $50$ GB for Fast-dLLM, representing a reduction of about $48\%$. This improvement arises from the flat, preallocated cache layout, which avoids dynamic memory allocation and the padding overhead associated with conventional four-dimensional KV-cache implementations.

\begin{figure}[t]
  \centering
  \begin{subfigure}[b]{0.49\linewidth}
    \centering
    \includegraphics[width=1\linewidth]{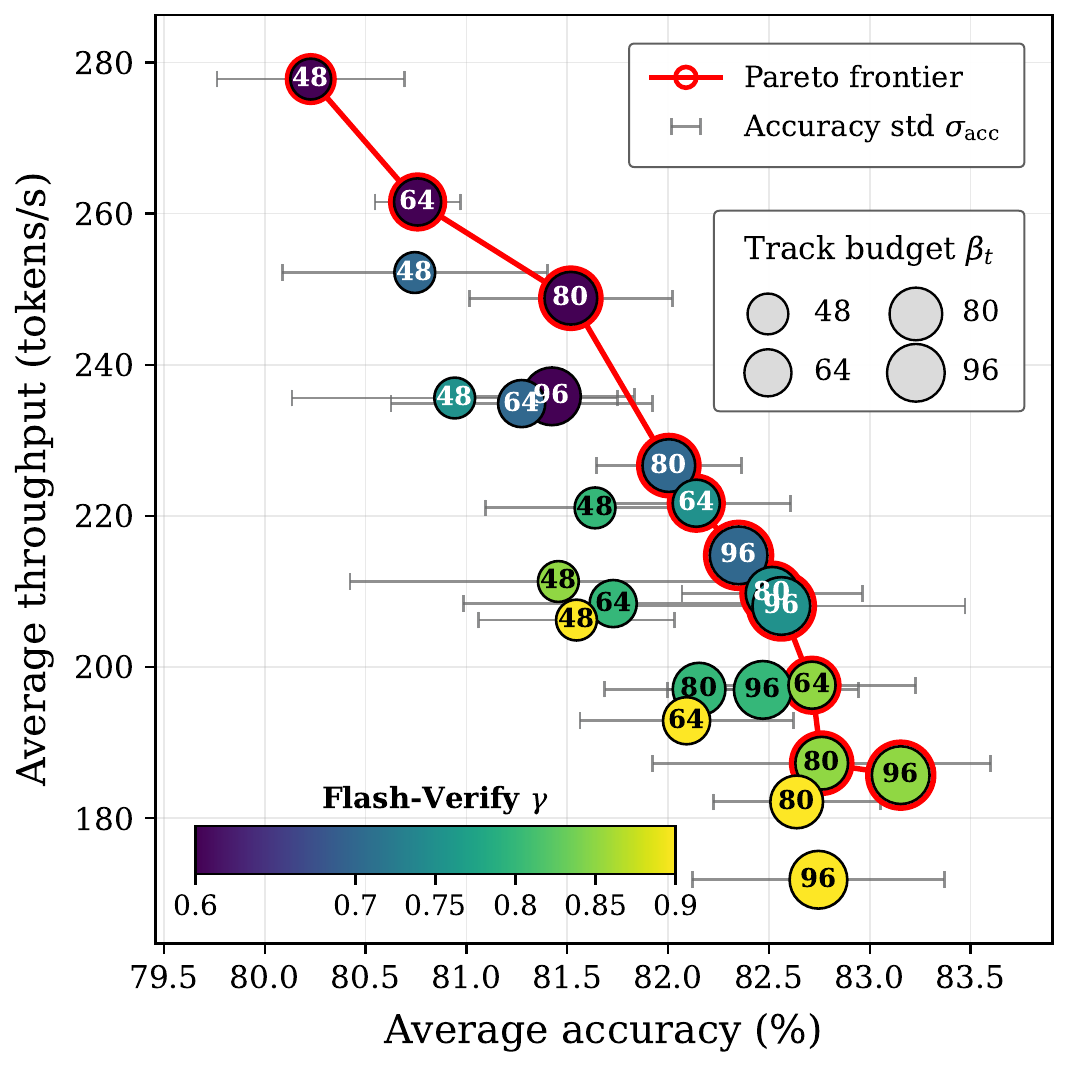}
    \vspace{-0.2in}
    \caption{Pareto frontier of Flash-Verify}
    \label{fig:pareto flash}
  \end{subfigure}
  \hfill
  \begin{subfigure}[b]{0.49\linewidth}
    \centering
    \includegraphics[width=1\linewidth]{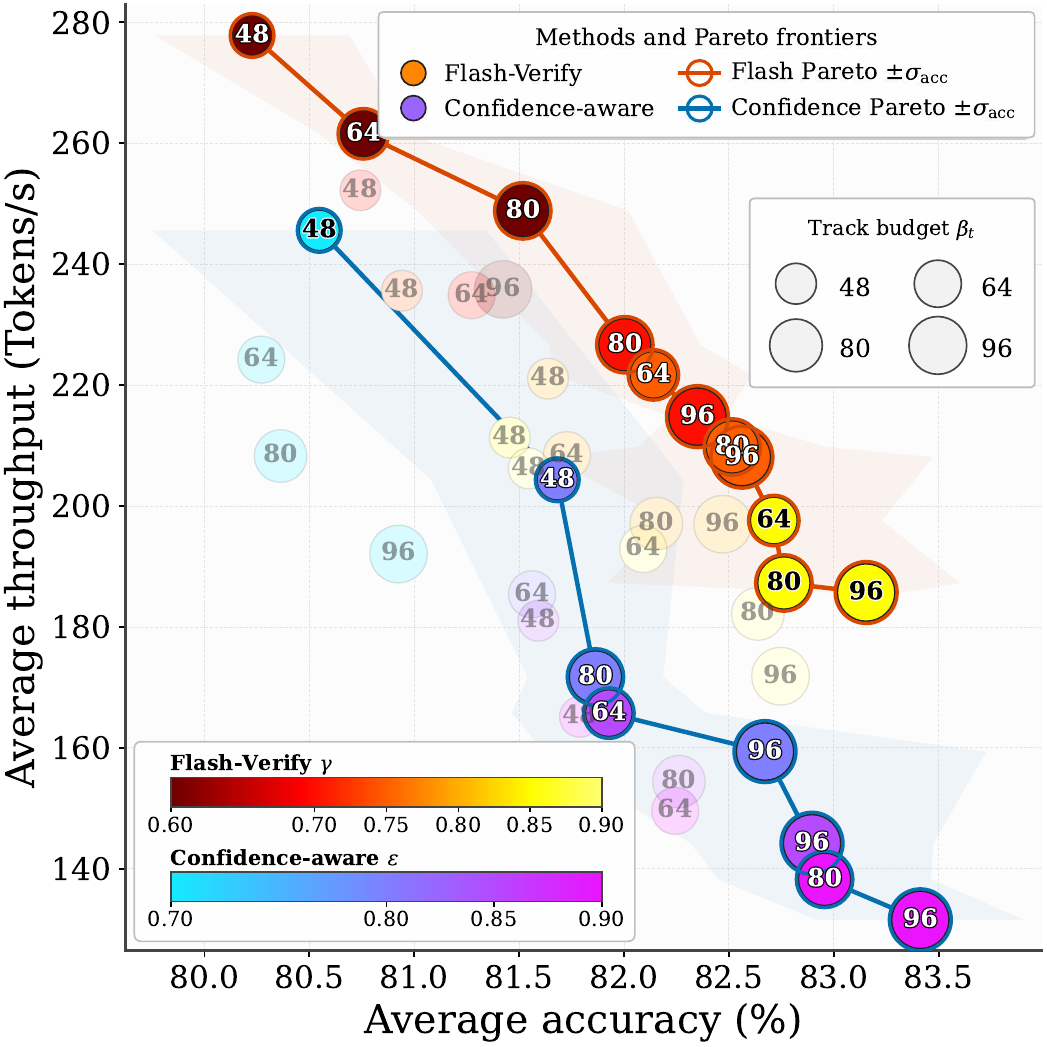}
    \vspace{-0.2in}
    \caption{Pareto frontiers of Flash-Verify and Confidence}
    \label{fig:pareto flash vs conf}
  \end{subfigure}
  \vspace{-0.08in}
  \caption{Accuracy–throughput trade-off under different track-budget values $\beta_t$ and denoising settings $\gamma$/$\epsilon$. The mean and standard deviation are computed over five random seeds.}
  \label{fig:fig}
  \vspace{-0.1in}
\end{figure}

\noindent\textbf{Accuracy--throughput trade-off.}
Figure~\ref{fig:pareto flash} characterizes the accuracy--throughput trade-off across track budgets $\beta_t$ and Flash-Verify thresholds $\gamma$. In general, accuracy and throughput are inversely related: increasing accuracy from approximately $80.2\%$ to $83.2\%$ reduces throughput from $278$ to $186$ tokens/s. Larger track budgets generally improve accuracy at the cost of additional computation, whereas $\gamma$ provides finer control over parallelism within a fixed budget. The Pareto frontier identifies the optimal operating points across these configurations. In particular, a favorable trade-off can be obtained by maintaining a relatively large $\beta_t$ while reducing $\gamma$ to promote greater decoding parallelism.

\noindent\textbf{Flash-Verify vs.\ confidence-aware decoding.}
Figure~\ref{fig:pareto flash vs conf} compares Flash-Verify with confidence-aware decoding. Across most of their shared accuracy range, Flash-Verify delivers substantially higher throughput. At approximately $82.6\%$--$82.9\%$ accuracy, it achieves $190$--$210$ tokens/s, compared with $140$--$160$ tokens/s for confidence-aware decoding. Although confidence-aware decoding attains a slightly higher peak accuracy of approximately $83.4\%$, its throughput decreases to about $131$ tokens/s. Overall, Flash-Verify provides a more favorable accuracy--throughput trade-off.

\begin{figure}[t]
  \centering
  \includegraphics[width=0.9\linewidth]{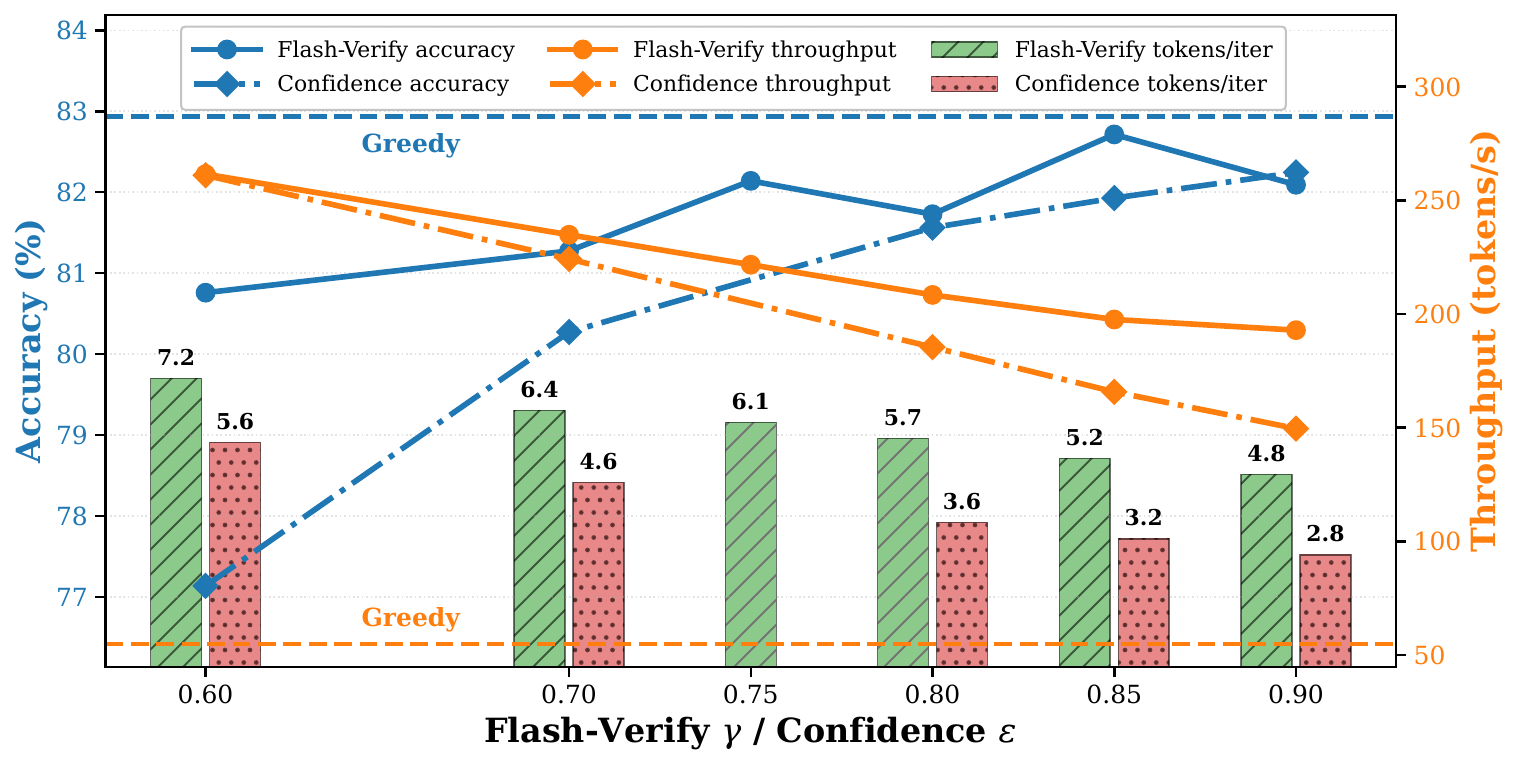}
  \caption{Tokens-per-iteration comparison of Flash-Verify and Confidence-aware decoding.}
  \label{fig:token per step}
\end{figure}

\noindent\textbf{Tokens decoded per iteration.}
Figure~\ref{fig:token per step} compares accuracy, throughput, and tokens decoded per iteration across Flash-Verify thresholds $\gamma$ and confidence-aware thresholds $\epsilon$. Across denoising settings, Flash-Verify generally decodes more tokens per iteration and achieves higher throughput, demonstrating greater parallel-decoding efficiency. Because Flash-Verify incurs an additional verification step, decoding $7.2$ tokens per iteration yields throughput comparable to confidence-aware decoding at $5.6$ tokens per iteration; however, Flash-Verify achieves $3.5\%$ higher accuracy. As the number of tokens decoded per iteration increases, the accuracy gap between the methods narrows, while Flash-Verify's throughput advantage grows to as much as $1.33\times$.

\begin{figure}[t]
  \centering
  \includegraphics[width=0.9\linewidth]{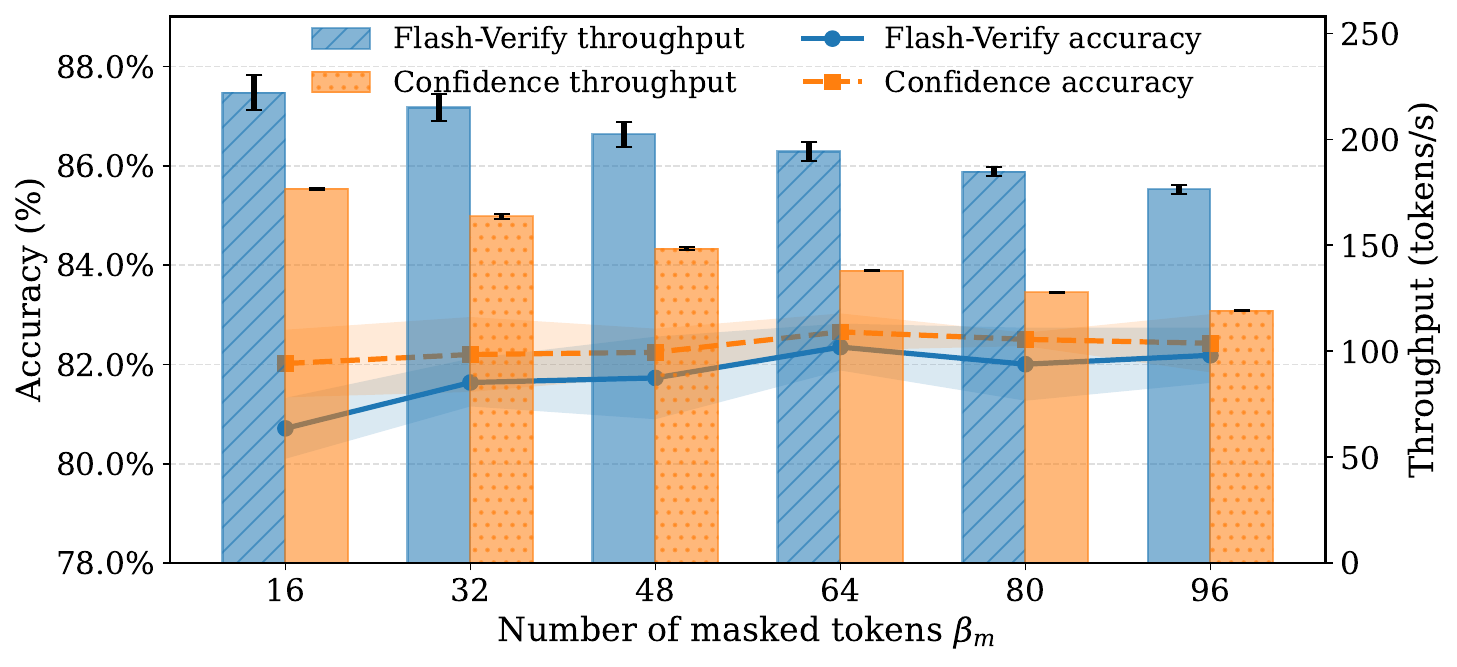}
  \caption{Impact of the masked-window size $\beta_m$ on accuracy and throughput}
  \label{fig:ablation mased window}
\end{figure}

\noindent\textbf{Masked-window size $\beta_m$.}
Figure~\ref{fig:ablation mased window} examines the effect of the masked-window size $\beta_m$ on the accuracy and throughput of Flash-Verify and confidence-aware decoding. For both methods, accuracy improves up to an intermediate window size, whereas throughput decreases as $\beta_m$ increases. Flash-Verify achieves comparable accuracy across all settings, with gaps of only $0.2$--$1.2$ percentage points. Moreover, it consistently delivers higher throughput, with its speedup over confidence-aware decoding increasing from approximately $1.4\times$ to $1.5\times$ as the window size grows.


\section{Related Work}
\noindent{\bf Diffusion Language Models and Acceleration.}
Masked diffusion models generate text by iteratively unmasking tokens predicted in parallel~\citep{li2025survey,austin2021structured,sahoo2024simple,shi2024simplified,zheng2024masked}. Scaling this paradigm has produced models competitive with autoregressive LLMs, including LLaDA~\citep{nie2025large}, Dream~\citep{ye2025dream}, and Gemini Diffusion~\citep{gemini_diffusion2025}. However, the lack of KV caching support makes dLLM inference slow: bidirectional attention and evolving hidden states across denoising steps prevent direct reuse of autoregressive caching strategies~\citep{pope2023efficiently}. Fast-dLLM~\citep{wu2025fast} introduces prefix caching with confidence-aware decoding. dKV-Cache~\citep{ma2025dkv} caches at fixed temporal intervals. Elastic-Cache~\citep{nguyen2025attention} makes updates adaptive via attention-pattern drift detection. FlashDLM~\citep{hu2025flashdlm} combines token-level change detection with an external autoregressive verifier. 

\noindent{\bf Parallel Decoding and Speculative Verification.}
Speculative decoding~\citep{leviathan2023fast,chen2023accelerating} accelerates autoregressive LLMs by drafting tokens with a fast model and verifying them in parallel with the target model. In diffusion LLMs, parallel decoding takes a different form: confidence-aware methods~\citep{wu2025fast} unmask all tokens above a threshold per step, Prophet~\citep{li2025diffusion} dynamically halts refinement when predictions stabilize, and FreeDave~\citep{wu2025free} drafts tokens in one forward pass then verifies in a second. FlashDLM~\citep{hu2025flashdlm} uses an external autoregressive LLM as the verifier. Our Flash-Verify differs from all of these: the dLLM serves as both drafter and verifier through the KV cache, with no external model and no training. We feed both the draft predictions and the original masks at the same positions under a causal attention mask, letting the model check its own consistency in a single additional forward pass. Tokens are accepted only when the draft and mask views agree, combining the throughput of aggressive parallel decoding with the accuracy of conservative single-token unmasking.

\section{Conclusion}

In this work, we proposed {\bf Flash-dLLM}, a training-free inference acceleration framework for diffusion Large Language Models. Flash-dLLM addresses the key efficiency bottlenecks in dLLM decoding by jointly handling IO-aware KV caching and parallel draft-and-verify decoding. By reducing redundant KV-cache read/write operations, prioritizing influential decoded tokens, and improving the confidence of early token commitments, Flash-dLLM effectively converts the inherent parallelism of dLLMs into practical inference speedup. Extensive experiments on mathematical reasoning and code generation demonstrate that Flash-dLLM achieves superior speed, memory efficiency, and scalability over existing acceleration methods, offering a promising and practical path toward efficient deployment of diffusion-based language models.

\section*{Acknowledgments}

This work is supported by the MBZUAI-WIS Joint Program for Artificial Intelligence Research.

\bibliography{references}
\bibliographystyle{plain}

\clearpage
\appendix

\section*{Appendix}
\renewcommand{\thesubsection}{\Alph{subsection}}

\noindent\textbf{Table of Contents}
\begin{itemize}[leftmargin=1.5em, itemsep=0pt, topsep=4pt]
    \item[\ref{app:limitations}] Limitations
    \item[\ref{app:broader-impact}] Broader Impact
    \item[\ref{app:flash-verify-theory}] Bounded-Deviation Guarantee for Flash-Verify
    \item[\ref{app:exp-setup}] Detailed Experiment Setup
    \item[\ref{app:stat-analysis}] Statistical analysis
    \item[\ref{app:ablation-generate-prefill}] Ablation of Generation and Prefill Length.
    \item[\ref{app:scalability}] More scalability analysis
    \item[\ref{app:sample}] Sample Response
\end{itemize}

\subsection{Limitations}
\label{app:limitations}

Flash-dLLM is evaluated on two representative masked diffusion LLMs across mathematical reasoning and code generation tasks. While the fused Triton kernel and Flash-Verify are architecture-agnostic in design, we have not yet validated them on continuous-space diffusion language models, where the cache update patterns may differ. Similarly, our benchmarks focus on structured-output tasks; the behavior of confidence-aware decoding on open-ended generation (e.g., long-form writing or dialogue), where token-level confidence distributions tend to be flatter, is an interesting direction for future work. The hyperparameters $\gamma$ and $\beta_m$ are fixed throughout generation; an adaptive scheme that adjusts these based on running confidence statistics could further improve the throughput-accuracy trade-off. We view these as natural extensions and a potential future research direction.

\subsection{Broader Impact}
\label{app:broader-impact}

This work targets inference-time efficiency for diffusion LLMs and does not introduce new training data, model architectures, or fine-tuning procedures. By reducing the computational and memory cost of dLLM decoding, Flash-dLLM lowers the hardware barrier to deploying these models, which could broaden access to non-autoregressive language generation for researchers and practitioners with limited GPU resources. Faster inference also reduces the energy consumption per generated sequence, contributing to more sustainable deployment of large language models. We do not modify safety filters or alignment mechanisms of the underlying models; the generation quality and any associated risks (e.g., producing harmful or biased content) are inherited from the base dLLM. We encourage users to apply standard content moderation and human oversight when deploying Flash-dLLM in user-facing applications, particularly in high-stakes domains.

\subsection{Bounded-Deviation Guarantee for Flash-Verify}
\label{app:flash-verify-theory}

\paragraph{Setup.}
We first fix the two distributions that Flash-Verify evaluates at a search
position, following the draft and verify passes of
Section~\ref{subsec:flash-draft-verify}, Figure~\ref{fig:overview} and
Algorithm~\ref{alg:flash-dllm}. Throughout, the quantity being controlled is
the gap between committing several positions at once and committing them from
the model's own chain-rule joint; this gap is the standard error term of
parallel unmasking in masked diffusion
models~\citep{wu2025fast,ben2025accelerated,ringel2026dependency}.

\emph{Draft pass.} At denoising step $t$, the draft pass runs the model on the
query set $\mathcal{Q}^t = \mathcal{T}^t \cup \mathcal{M}^t_{\beta_m}$ with every
position in $\mathcal{M}^t_{\beta_m}$ held at $[\mathrm{MASK}]$. Let $\mathbf{x}^t$
denote this input sequence. For each $i \in \mathcal{M}^t_{\beta_m}$ the draft
pass yields the draft-pass distribution over the vocabulary $V$, which we
call the \emph{draft marginal}
\begin{equation}
q_i(x) := p_\theta\!\big(x^i = x \mid \mathbf{x}^t\big), \qquad x \in V,
\end{equation}
from which the draft prediction and its confidence are
$\hat{x}^i = \arg\max_{x} q_i(x)$ and $c^i = q_i(\hat{x}^i)$. The window is
partitioned by confidence into the confident set
$\mathcal{D}^t = \{i : c^i \ge \epsilon\}$, which is accepted directly, and the
search set $\mathcal{S}^t = \mathcal{M}^t_{\beta_m} \setminus \mathcal{D}^t$,
which is sent to verification. Note that $q_i$ conditions on no other prediction
made at step $t$: every position in $\mathcal{M}^t_{\beta_m}$, including
$\mathcal{D}^t$ and the other search positions, is masked in $\mathbf{x}^t$.
Committing $\mathcal{D}^t$ directly is covered by the analysis of confidence-aware
decoding~\citep[Thm.~1]{wu2025fast}: if $n$ positions all have marginal
confidence above $1-\varepsilon$ with $(n+1)\varepsilon \le 1$, the joint mode
equals the product-of-marginals mode and the two distributions are within
$\tfrac{3n-1}{2}\varepsilon$ in total variation. That result says nothing about
$\mathcal{S}^t$, whose positions fail the confidence condition by construction,
and it is for those positions that the verify pass and the guarantee below are
needed.

\emph{Verify pass.} Let $\mathbf{x}^t_{+}$ denote $\mathbf{x}^t$ with the confident
positions committed to their drafts, $x^{\mathcal{D}^t} = \hat{x}^{\mathcal{D}^t}$;
together with the previously decoded tokens this forms the tracked context
$\mathcal{T}_v$. Write the search positions in causality order (draft
confidence, descending) as $i_1, \dots, i_K$ with $K = |\mathcal{S}^t|$. The
verify query duplicates $\mathcal{S}^t$ into a \emph{draft view}, filled with
$\hat{x}^{\mathcal{S}^t}$, and a \emph{mask view}, filled with $[\mathrm{MASK}]$.
Under the causal verify mask of Figure~\ref{fig:overview}, the mask view at
$i_k$ attends to $\mathcal{T}_v$, to the draft view at $i_1, \dots, i_{k-1}$ and
to the mask view at $i_k, \dots, i_K$; it does not attend to the draft view at
$i_k$ or at any later position. Every search position therefore enters its
context exactly once, as a draft if it precedes $i_k$ and as $[\mathrm{MASK}]$
otherwise, so the model output at the mask view of $i_k$ is the mask-view
distribution, which we call the \emph{verify conditional}
\begin{equation}
\label{eq:verify-conditional}
p_{i_k}(x) := p_\theta\!\big(x^{i_k} = x \mid \mathbf{x}^t_{+},\
\hat{x}^{i_1}, \dots, \hat{x}^{i_{k-1}}\big), \qquad x \in V,
\end{equation}
i.e.\ the model's prediction for $i_k$ when the positions before it in
causality order are committed to their drafts and $i_k, \dots, i_K$ remain
masked. This is exactly the input that sequential decoding in this order would
present at its $k$-th step, and it is the same object that speculative decoding
verifies against: the target model's conditional at a drafted position given
the accepted tokens before
it~\citep{leviathan2023fast,chen2023accelerating,yin2024theoretical}.
The mask-view prediction and confidence are
$\tilde{x}^{i} = \arg\max_{x} p_{i}(x)$ and $\tilde{c}^{i} = p_{i}(\tilde{x}^{i})$.
Flash-Verify accepts position $i$ according to
\begin{equation}
\mathrm{accept}(i)
= \mathbb{I}\!\big[\hat{x}^i = \tilde{x}^i\big]\cdot
  \mathbb{I}\!\big[\tilde{c}^i \ge \gamma\big],
\end{equation}
applied in causality order and stopped at the first rejection, so the
committed block is the longest prefix $\mathcal{A} = \{i_1, \dots, i_m\}$ on
which every position is accepted. An accepted token is simultaneously the mode
of the draft marginal $q_i$ and a $\gamma$-confident mode of the verify
conditional $p_i$. On acceptance the model commits the agreed token, and we
write the delivered distribution at position $i$ as the point mass
$\pi_i(x) := \mathbb{I}[x = \tilde{x}^i]$. For two distributions $u, v$ over a
finite set, total variation distance is
$\mathrm{TV}(u, v) := \tfrac{1}{2}\sum_{x} |u(x) - v(x)|$.

\begin{assumption}
\label{assump:reference}
The reference for the committed block is the model's own sequential
distribution in causality order: at position $i_k$ it is the verify conditional
$p_{i_k}$ of Eq.~\eqref{eq:verify-conditional}, and over the block it is the
chain-rule joint
$p_{\mathcal{A}}(x^{i_1}, \dots, x^{i_m}) := \prod_{k=1}^{m}
p_\theta\!\big(x^{i_k} \mid \mathbf{x}^t_{+}, x^{i_1}, \dots, x^{i_{k-1}}\big)$,
which is the distribution one-token-per-step decoding of the same positions in
the same order samples from. Deviation is measured against this model
distribution rather than the data distribution.
\end{assumption}

\begin{theorem}[Per-token deviation]
\label{thm:per-token}
Under Assumption~\ref{assump:reference}, for every $i \in \mathcal{A}$,
\begin{equation}
\mathrm{TV}(\pi_i, p_i) = 1 - \tilde{c}^i \le 1 - \gamma.
\end{equation}
\end{theorem}

\begin{proof}
Let $m := \tilde{x}^i = \arg\max_x p_i(x)$, so that $p_i(m) = \tilde{c}^i$ and
$\pi_i(x) = \mathbb{I}[x = m]$. By the definition of total variation,
\begin{equation}
\mathrm{TV}(\pi_i, p_i)
= \tfrac{1}{2}\sum_{x \in V} \big|\pi_i(x) - p_i(x)\big|
= \tfrac{1}{2}\Big(\big|\pi_i(m) - p_i(m)\big|
  + \sum_{x \ne m} \big|\pi_i(x) - p_i(x)\big|\Big).
\end{equation}
At $x = m$ we have $\pi_i(m) = 1$, hence
$|\pi_i(m) - p_i(m)| = 1 - p_i(m)$. At every $x \ne m$ we have $\pi_i(x) = 0$,
hence $|\pi_i(x) - p_i(x)| = p_i(x)$, and since $p_i$ is a probability
distribution, $\sum_{x \ne m} p_i(x) = 1 - p_i(m)$. Combining,
\begin{equation}
\mathrm{TV}(\pi_i, p_i)
= \tfrac{1}{2}\big[(1 - p_i(m)) + (1 - p_i(m))\big]
= 1 - p_i(m)
= 1 - \tilde{c}^i.
\end{equation}
The acceptance rule fires only when $\tilde{c}^i \ge \gamma$, therefore
$\mathrm{TV}(\pi_i, p_i) = 1 - \tilde{c}^i \le 1 - \gamma$.
\end{proof}

\begin{corollary}[Block deviation]
\label{cor:block}
Let $\mathcal{A} = \{i_1, \dots, i_m\}$ be the block committed by a single
verify pass, $\pi_{\mathcal{A}} := \prod_{k=1}^{m} \pi_{i_k}$ the delivered
distribution over $\mathcal{A}$, and $p_{\mathcal{A}}$ the chain-rule joint of
Assumption~\ref{assump:reference}. Then
\begin{equation}
\mathrm{TV}(\pi_{\mathcal{A}}, p_{\mathcal{A}})
= 1 - \prod_{k=1}^{m} \tilde{c}^{\,i_k}
\le 1 - \gamma^{m}
\le m\,(1 - \gamma).
\end{equation}
\end{corollary}

\begin{proof}
Every accepted token equals its draft, $\tilde{x}^{i_k} = \hat{x}^{i_k}$, so
$\pi_{\mathcal{A}}$ is the point mass at $\hat{x}^{\mathcal{A}} =
(\hat{x}^{i_1}, \dots, \hat{x}^{i_m})$. The computation in the proof of
Theorem~\ref{thm:per-token} applies verbatim to any point mass $\delta_z$ and any
distribution $p$ on a finite set and gives $\mathrm{TV}(\delta_z, p) = 1 - p(z)$;
hence $\mathrm{TV}(\pi_{\mathcal{A}}, p_{\mathcal{A}}) = 1 -
p_{\mathcal{A}}(\hat{x}^{\mathcal{A}})$. Evaluating the chain rule at
$\hat{x}^{\mathcal{A}}$, the $k$-th factor is
$p_\theta(\hat{x}^{i_k} \mid \mathbf{x}^t_{+}, \hat{x}^{i_1}, \dots, \hat{x}^{i_{k-1}})
= p_{i_k}(\hat{x}^{i_k}) = p_{i_k}(\tilde{x}^{i_k}) = \tilde{c}^{\,i_k}$ by
Eq.~\eqref{eq:verify-conditional}; the conditioning tokens $\hat{x}^{i_1}, \dots,
\hat{x}^{i_{k-1}}$ all lie in the accepted prefix, so no rejected draft enters
any factor. Thus $p_{\mathcal{A}}(\hat{x}^{\mathcal{A}}) = \prod_k \tilde{c}^{\,i_k}
\ge \gamma^{m}$, and $1 - \gamma^{m} \le m(1-\gamma)$ by Bernoulli's inequality.
\end{proof}

\begin{remark}[What the reference is]
Because the verify mask is causal in causality order, $p_{\mathcal{A}}$ is the
model's genuine joint over the block, not a product of marginals: the verified
tokens incur no parallel-unmasking dependence error, and the bound is an
equality in the mask-view confidences. The dependence term bounded in KL
by~\citep{ben2025accelerated} and in total variation
by~\citep{ringel2026dependency} therefore concerns only $\mathcal{D}^t$, which is
committed from marginals and handled by~\citep[Thm.~1]{wu2025fast}. Stopping
at the first rejection is what keeps $\mathcal{A}$ a prefix, so that every
conditional in $p_{\mathcal{A}}$ conditions only on accepted tokens.
\end{remark}

\begin{remark}[Why $\gamma$ gates the mask view]
The bound is controlled by $\gamma$ because acceptance thresholds the
confidence of $p_i$, the distribution the token is actually committed from.
The confidence threshold $\epsilon$ instead gates $c^i = q_i(\hat{x}^i)$,
computed with all of $\mathcal{M}^t_{\beta_m}$ masked~\citep{wu2025fast};
$q_i$ ignores the other positions decoded in the same step, so $\hat{x}^i$ can
be its mode yet carry little mass under $p_i$, and for a search token
$c^i < \epsilon$ by construction. Hence lowering $\epsilon$ is not equivalent to
verifying: only the latter carries the $1-\gamma$ guarantee against $p_i$. In
the language of speculative decoding, the rule is a deterministic, biased
acceptance~\citep{yin2024theoretical}: the agreement test alone is lossless
with respect to greedy decoding from $p_i$, as in the auto-speculative verifier
of~\citep{agrawal2025structuring}, while $\gamma$ trades a bounded distribution
bias of $1-\tilde{c}^i$ for more accepted tokens, the trade-off swept in
Figure~\ref{fig:pareto flash}.
\end{remark}

\subsection{Detailed Experiment Setup}
\label{app:exp-setup}
\noindent{\bf Implementation Details.} All experiments run on a single NVIDIA A100 80GB GPU. We evaluate \textbf{Flash-dLLM} on LLaDA-1.5~\citep{zhu2025llada} across GSM8K~\citep{cobbe2021training}, MATH~\citep{hendrycks2021measuring}, HumanEval~\citep{chen2021evaluating}, and MBPP~\citep{austin2021program}. We implement the fused KV-cache kernel in Triton 2.0. Default benchmark is GSM8K, with default hyperparameters: confidence threshold $\epsilon=0.9$, verify threshold $\gamma=0.8$, block size $\beta=16$, tracked budget $\beta_t=80$, sliding window size $\beta_m=64$, generation length 512. For fair comparison, we re-run all baselines under identical hardware and software configurations. \noindent{\bf Baselines.} We compare against three approaches: (1)~\textbf{No Cache}: standard dLLM inference without KV caching, under both greedy (fixed-step) and confidence-aware decoding; (2)~\textbf{Fast-dLLM}~\citep{wu2025fast}: prefix-caching with confidence-aware decoding; (3)~\textbf{Elastic-Cache}~\citep{nguyen2025attention}: adaptive KV caching with attention-pattern-based cache reuse. We report Flash-dLLM results under three configurations: greedy decoding (pure KV-cache speedup), confidence-aware decoding (KV-cache + parallel decoding), and Flash-Verify (KV-cache + draft-and-verify parallel decoding). \noindent{\bf Evaluation Metrics.} We use \texttt{lm-eval-harness}~\citep{eval-harness}. Throughput is measured as decoding tokens/sec averaged over the benchmark, following Fast-dLLM's protocol~\citep{wu2025fast}. Accuracy metrics: GSM8K uses 5-shot \texttt{flexible\_extract}; MATH uses 4-shot \texttt{math\_verify}; HumanEval uses 0-shot \texttt{pass@1} with Fast-dLLM post-processing; MBPP uses 3-shot \texttt{pass@1}. We test at generation lengths 256 and 512 to study scalability.

\noindent{\bf Evaluation Framework and Metrics.}
Our evaluation protocol comprehensively assesses both inference efficiency and model performance across various tasks. To ensure standardization and reproducibility, we conduct all task-specific evaluations using the \texttt{lm-eval-harness} library~\citep{eval-harness}. Inference speed is measured by throughput in tokens per second (t/s), calculated as the average number of tokens generated by the model over the entire sequence until it produces an end-of-sequence (\textbf{\texttt{<eos>}}) token. We maintain consistency with Fast-dLLM~\citep{wu2025fast} in our calculation methodology to ensure comparable speed benchmarks. Task-specific performance is evaluated using established metrics suitable for each benchmark. For GSM8K~\citep{cobbe2021training}, we report 5-shot \texttt{flexible\_extract} exact match accuracy. For the MATH dataset~\citep{hendrycks2021measuring}, we report the 4-shot \texttt{math\_verify} score using the \texttt{minerva\_math} variant. For HumanEval~\citep{chen2021evaluating}, we evaluate 0-shot accuracy using a post-processing script consistent with the Fast-dLLM implementation to ensure fair comparison. Finally, for MBPP~\citep{austin2021program}, we report the 3-shot \texttt{pass@1} metric.

\noindent{\bf Hyper-parameters:} The hyper-parameters used for Flash-dLLM are presented in Table \ref{tab:hyperparameters}. Specifically, 
\begin{itemize}
    \item We set the basic block size to $\beta=16$, corresponding to the smallest block used in our Triton implementation.
    \item For confidence-aware decoding, we use $\epsilon=0.9$, following the optimal setting reported by Fast-dLLM.
    \item We use a default masked-window size of $\beta_m=64$ across all settings to balance accuracy and decoding speed.
    \item We set $\gamma=0.8$ for GSM8K and MBPP and $\gamma=0.85$ for MATH and HumanEval.
    \item We use a batch size of $32$ in all experiments to maximize GPU utilization.
\end{itemize}

\begin{table}[t]
    \centering
    \caption{The hyper-parameters of Flash-dLLM under various settings.}
    \label{tab:hyperparameters}
    \resizebox{0.99\linewidth}{!}{
    \begin{tabular}{llccccc}
        \toprule
        \textbf{Model} & \textbf{Benchmark} & \textbf{Gen Length} & Tracking budget $\beta_t$ & Window size $\beta_m$ & Flash-Verify $\gamma$ & Batch size \\
        \midrule
        \multirow{9}{*}{\textbf{LLaDA-1.5}} & \multirow{2}{*}{GSM8K (5-shot)} & 256 & 64 & 64 & 0.8 & 32 \\
         & & 512 & 64 & 64 & 0.8 & 32\\
        \cmidrule(l){2-7}
         & \multirow{2}{*}{MATH (4-shot)} & 256 & 64 & 64 & 0.85 & 32 \\
         & & 512 & 64 & 64 & 0.85 & 32 \\
        \cmidrule(l){2-7}
         & \multirow{2}{*}{Humaneval (0-shot)} & 256 & 64 & 64 & 0.85 & 32 \\
         & & 512 & 64 & 64 & 0.85 & 32 \\
        \cmidrule(l){2-7}
         & \multirow{2}{*}{MBPP (3-shot)} & 256 & 48 & 64 & 0.8 & 32 \\
         & & 512 & 48 & 64 & 0.8 & 32 \\
        
        \bottomrule
    \end{tabular}
    }
\end{table}

\subsection{Statistical analysis}
\label{app:stat-analysis}

\newcommand{\resultcell}[4]{%
  \makecell[c]{%
    $#1$\,{\footnotesize $\pm\,#2$}\\[-1pt]
    \textcolor{blue}{$#3$\,{\footnotesize $\pm\,#4$}}%
  }%
}

\begin{table}[t]
  \centering
  \caption{
    Mean accuracy (\%) and throughput over five random seeds.
    Accuracy is shown on the first line and throughput is shown
    in blue on the second line. Values are mean $\pm$ standard deviation.
  }
  \label{tab:gamma-tracknum-results}

  \setlength{\tabcolsep}{7pt}
  \renewcommand{\arraystretch}{1.15}

  \resizebox{0.6\linewidth}{!}{%
    \begin{tabular}{c c c c c}
      \toprule
      & \multicolumn{4}{c}{Track budget $\beta_t$} \\
      \cmidrule(lr){2-5}
      $\gamma$
        & $48$
        & $64$
        & $80$
        & $96$ \\
      \midrule

      $0.60$
        & \resultcell{80.23}{0.46}{277.80}{1.18}
        & \resultcell{80.76}{0.21}{261.54}{1.00}
        & \resultcell{81.52}{0.50}{248.80}{0.70}
        & \resultcell{81.42}{0.41}{235.82}{0.76}
        \\

      $0.70$
        & \resultcell{80.74}{0.66}{252.20}{2.52}
        & \resultcell{81.27}{0.65}{234.86}{4.94}
        & \resultcell{82.00}{0.36}{226.66}{3.28}
        & \resultcell{82.35}{0.13}{214.78}{2.18}
        \\

      $0.75$
        & \resultcell{80.94}{0.81}{235.60}{1.36}
        & \resultcell{82.14}{0.47}{221.68}{0.71}
        & \resultcell{82.52}{0.45}{209.76}{1.29}
        & \resultcell{82.56}{0.91}{208.08}{0.51}
        \\

      $0.80$
        & \resultcell{81.64}{0.54}{221.06}{3.58}
        & \resultcell{81.73}{0.74}{208.40}{2.64}
        & \resultcell{82.15}{0.47}{197.08}{1.54}
        & \resultcell{82.47}{0.47}{196.98}{0.84}
        \\

      $0.85$
        & \resultcell{81.46}{1.03}{211.32}{0.69}
        & \resultcell{82.71}{0.51}{197.60}{0.70}
        & \resultcell{82.76}{0.84}{187.28}{0.75}
        & \resultcell{83.15}{0.17}{185.70}{0.61}
        \\

      $0.90$
        & \resultcell{81.55}{0.49}{206.22}{0.50}
        & \resultcell{82.09}{0.53}{192.90}{0.75}
        & \resultcell{82.64}{0.41}{182.16}{0.68}
        & \resultcell{82.75}{0.62}{171.86}{0.33}
        \\

      $1.00$
        & \resultcell{81.79}{0.64}{165.14}{0.83}
        & \resultcell{82.24}{0.43}{149.60}{0.47}
        & \resultcell{82.96}{0.51}{138.28}{0.37}
        & \resultcell{83.41}{0.50}{131.58}{1.41}
        \\

      \bottomrule
    \end{tabular}%
  }
\end{table}

Table ~\ref{tab:gamma-tracknum-results} examines the joint effect of the denoising parameter $\gamma$ and track budget $\beta_t$ on accuracy and throughput. The results show a consistent accuracy--efficiency trade-off: increasing either parameter generally improves accuracy while reducing decoding throughput. At a fixed $\gamma$, increasing $\beta_t$ from $48$ to $96$ improves accuracy by approximately $0.83$--$1.69$ percentage points, but decreases throughput by about $10.9\%$--$20.3\%$. Similarly, increasing $\gamma$ from $0.60$ to $1.00$ yields an accuracy gain of $1.44$--$1.99$ percentage points across the evaluated track budgets, accompanied by a throughput reduction of approximately $40.6\%$--$44.4\%$. Although a few neighboring configurations exhibit small non-monotonic variations, these differences are generally comparable to the reported standard deviations and do not alter the overall trend.

The highest-throughput configuration is obtained with $\gamma=0.60$ and $\beta_t=48$, reaching $277.80$ tokens/s at $80.23\%$ accuracy. In contrast, the highest accuracy of $83.41\%$ is achieved with $\gamma=1.00$ and $\beta_t=96$, where throughput decreases to $131.58$ tokens/s. Thus, moving from the fastest to the most accurate configuration improves accuracy by $3.18$ percentage points while reducing throughput by approximately $52.6\%$. Intermediate settings provide more balanced operating points. In particular, $\gamma=0.85$ and $\beta_t=96$ achieve $83.15\%$ accuracy at $185.70$ tokens/s, remaining only $0.26$ percentage points below the maximum accuracy while providing approximately $41.1\%$ higher throughput. The relatively small standard deviations across five random seeds---at most $1.03$ percentage points for accuracy and $4.94$ tokens/s for throughput---also indicate that the observed trade-off is consistent across runs. Overall, $\gamma$ and $\beta_t$ offer complementary controls for selecting either a throughput-oriented or accuracy-oriented operating point.

\subsection{Ablation of Generation and Prefill Length.}
\label{app:ablation-generate-prefill}
\begin{table*}[t]
\centering
\caption{Ablation of generation and prefill lengths.}
\label{tab:ablation length and prefill}
\scriptsize
\setlength{\tabcolsep}{2.5pt}
\renewcommand{\arraystretch}{1.05}

\begin{subtable}[t]{0.49\textwidth}
\centering
\caption{Impact of generation length on LLaDA-1.5 for GSM8K (5-shot) with a batch size of 16.}
\label{tab:ablation length}
\resizebox{\linewidth}{!}{%
\begin{tabular}{c|c|cccc}
\toprule
\bfseries Approach &
\bfseries Metric &
\bfseries Len. 128 &
\bfseries Len. 256 &
\bfseries Len. 512 &
\bfseries Len. 1024 \\
\midrule

\multirow{3}{*}{\makecell{Flash-Cache\\+ Confidence}}
& Accuracy
& \makecell{79.03}
& \makecell{82.36}
& \makecell{82.24}
& \makecell{82.29} \\

& Throughput
& \makecell{\textcolor{blue}{126.7}
  (\textcolor{orange}{1.00$\times$})}
& \makecell{\textcolor{blue}{140.4}
  (\textcolor{orange}{1.00$\times$})}
& \makecell{\textcolor{blue}{131.8}
  (\textcolor{orange}{1.00$\times$})}
& \makecell{\textcolor{blue}{125.7}
  (\textcolor{orange}{1.00$\times$})} \\

& Tokens/step
& \makecell{\textcolor{red}{2.5}}
& \makecell{\textcolor{red}{2.8}}
& \makecell{\textcolor{red}{2.8}}
& \makecell{\textcolor{red}{2.9}} \\

\midrule

\multirow{3}{*}{\makecell{Flash-Cache\\+ Flash-Verify}}
& Accuracy
& \makecell{79.55}
& \makecell{81.90}
& \makecell{81.73}
& \makecell{81.72} \\

& Throughput
& \makecell{\textcolor{blue}{174.5}
  (\textcolor{orange}{1.38$\times$})}
& \makecell{\textcolor{blue}{198.4}
  (\textcolor{orange}{1.41$\times$})}
& \makecell{\textcolor{blue}{186.2}
  (\textcolor{orange}{1.41$\times$})}
& \makecell{\textcolor{blue}{173.0}
  (\textcolor{orange}{1.38$\times$})} \\

& Tokens/step
& \makecell{\textcolor{red}{5.6}}
& \makecell{\textcolor{red}{5.6}}
& \makecell{\textcolor{red}{5.6}}
& \makecell{\textcolor{red}{5.6}} \\

\bottomrule
\end{tabular}%
}
\end{subtable}
\hfill
\begin{subtable}[t]{0.49\textwidth}
\centering
\caption{Impact of prefill length on LLaDA-1.5 for GSM8K-512 with a batch size of 16.}
\label{tab:ablation prefill}
\resizebox{\linewidth}{!}{%
\begin{tabular}{c|c|cccc}
\toprule
\bfseries Approach &
\bfseries Metric &
\bfseries 1-shot &
\bfseries 3-shot &
\bfseries 5-shot &
\bfseries 8-shot \\
\midrule

\multirow{3}{*}{\makecell{Flash-Cache\\+ Confidence}}
& Accuracy
& \makecell{80.00}
& \makecell{82.27}
& \makecell{82.24}
& \makecell{82.25} \\

& Throughput
& \makecell{\textcolor{blue}{155.8}
  (\textcolor{orange}{1.00$\times$})}
& \makecell{\textcolor{blue}{145.4}
  (\textcolor{orange}{1.00$\times$})}
& \makecell{\textcolor{blue}{131.8}
  (\textcolor{orange}{1.00$\times$})}
& \makecell{\textcolor{blue}{123.1}
  (\textcolor{orange}{1.00$\times$})} \\

& Tokens/step
& \makecell{\textcolor{red}{2.8}}
& \makecell{\textcolor{red}{2.8}}
& \makecell{\textcolor{red}{2.8}}
& \makecell{\textcolor{red}{2.8}} \\

\midrule

\multirow{3}{*}{\makecell{Flash-Cache\\+ Flash-Verify}}
& Accuracy
& \makecell{80.12}
& \makecell{81.73}
& \makecell{81.73}
& \makecell{81.73} \\

& Throughput
& \makecell{\textcolor{blue}{234.4}
  (\textcolor{orange}{1.50$\times$})}
& \makecell{\textcolor{blue}{209.6}
  (\textcolor{orange}{1.44$\times$})}
& \makecell{\textcolor{blue}{186.2}
  (\textcolor{orange}{1.41$\times$})}
& \makecell{\textcolor{blue}{165.7}
  (\textcolor{orange}{1.35$\times$})} \\

& Tokens/step
& \makecell{\textcolor{red}{5.8}}
& \makecell{\textcolor{red}{5.7}}
& \makecell{\textcolor{red}{5.7}}
& \makecell{\textcolor{red}{5.6}} \\

\bottomrule
\end{tabular}%
}
\end{subtable}

\end{table*}

Table~\ref{tab:ablation length and prefill} examines the sensitivity of the two decoding strategies to generation length and prefill length. As shown in Table~\ref{tab:ablation length}, Flash-Cache with Flash-Verify consistently provides higher throughput than confidence-aware decoding across all generation lengths. Its throughput advantage ranges from $1.38\times$ to $1.41\times$, corresponding to absolute improvements of $47.3$--$58.0$ tokens/s. Both methods reach their highest throughput at a generation length of 256, attaining $140.4$ tokens/s for confidence-aware decoding and $198.4$ tokens/s for Flash-Verify. Beyond 256 tokens, throughput gradually decreases, while accuracy remains nearly unchanged. For example, increasing the generation length from 256 to 1024 reduces throughput by approximately $10.5\%$ for confidence-aware decoding and $12.8\%$ for Flash-Verify, without providing a corresponding accuracy improvement.

Table~\ref{tab:ablation prefill} shows that increasing the number of in-context examples, and consequently the prefill length, reduces throughput for both methods. From 1-shot to 8-shot prompting, confidence-aware throughput decreases from $155.8$ to $123.1$ tokens/s, a reduction of approximately $21.0\%$. Flash-Verify decreases from $234.4$ to $165.7$ tokens/s, corresponding to a reduction of approximately $29.3\%$. Nevertheless, Flash-Verify remains faster in every setting, providing speedups of $1.50\times$, $1.44\times$, $1.41\times$, and $1.35\times$ for the 1-, 3-, 5-, and 8-shot settings, respectively. The decreasing relative speedup suggests that the computational cost of processing a longer prefix affects Flash-Verify more strongly, although it retains an absolute advantage of $42.6$--$78.6$ tokens/s.

\subsection{More scalability analysis}
\label{app:scalability}

\begin{table}[t]
\centering
\caption{Effect of batch size on throughput. Using LLaDA-1.5 for GSM8K-512 (5-shot). Speedup is measured relative to Flash-Cache with greedy decoding at the same batch size.}
\label{tab:ablation-bs-speed}
\resizebox{0.99\linewidth}{!}{%
\begin{tabular}{c|cc|ccccccc}
\toprule
\multirow{2}{*}{\textbf{Approach}}
& \multirow{2}{*}{\textbf{ACC}}
& \multirow{2}{*}{\textbf{Tok./step}}
& \multicolumn{7}{c}{\textbf{Throughput}} \\
\cmidrule(lr){4-10}
& &
& \textbf{B=1}
& \textbf{B=2}
& \textbf{B=4}
& \textbf{B=8}
& \textbf{B=16}
& \textbf{B=24}
& \textbf{B=32} \\
\midrule

\makecell{\bfseries Flash-Cache\\\bfseries + Greedy}
& \makecell{82.94}
& \makecell{\textcolor{red}{1.0}}
& \makecell{{\footnotesize \textcolor{blue}{18.5}
  (\textcolor{orange}{1.0$\times$})}}
& \makecell{{\footnotesize \textcolor{blue}{28.9}
  (\textcolor{orange}{1.0$\times$})}}
& \makecell{{\footnotesize \textcolor{blue}{37.7}
  (\textcolor{orange}{1.0$\times$})}}
& \makecell{{\footnotesize \textcolor{blue}{46.3}
  (\textcolor{orange}{1.0$\times$})}}
& \makecell{{\footnotesize \textcolor{blue}{51.3}
  (\textcolor{orange}{1.0$\times$})}}
& \makecell{{\footnotesize \textcolor{blue}{53.6}
  (\textcolor{orange}{1.0$\times$})}}
& \makecell{{\footnotesize \textcolor{blue}{55.0}
  (\textcolor{orange}{1.0$\times$})}}
\\

\midrule

\makecell{\bfseries Flash-Cache\\\bfseries + Confidence}
& \makecell{82.87}
& \makecell{\textcolor{red}{2.8}}
& \makecell{{\footnotesize \textcolor{blue}{51.7}
  (\textcolor{orange}{2.8$\times$})}}
& \makecell{{\footnotesize \textcolor{blue}{78.0}
  (\textcolor{orange}{2.7$\times$})}}
& \makecell{{\footnotesize \textcolor{blue}{102.4}
  (\textcolor{orange}{2.7$\times$})}}
& \makecell{{\footnotesize \textcolor{blue}{120.5}
  (\textcolor{orange}{2.6$\times$})}}
& \makecell{{\footnotesize \textcolor{blue}{131.8}
  (\textcolor{orange}{2.6$\times$})}}
& \makecell{{\footnotesize \textcolor{blue}{136.4}
  (\textcolor{orange}{2.5$\times$})}}
& \makecell{{\footnotesize \textcolor{blue}{139.5}
  (\textcolor{orange}{2.5$\times$})}}
\\

\midrule

\makecell{\bfseries Flash-Cache\\\bfseries + Flash-Verify}
& \makecell{\textbf{83.02}}
& \makecell{\textcolor{red}{5.7}}
& \makecell{{\footnotesize \textcolor{blue}{56.0}
  (\textcolor{orange}{3.0$\times$})}}
& \makecell{{\footnotesize \textcolor{blue}{90.9}
  (\textcolor{orange}{3.1$\times$})}}
& \makecell{{\footnotesize \textcolor{blue}{131.3}
  (\textcolor{orange}{3.5$\times$})}}
& \makecell{{\footnotesize \textcolor{blue}{164.8}
  (\textcolor{orange}{3.6$\times$})}}
& \makecell{{\footnotesize \textcolor{blue}{186.2}
  (\textcolor{orange}{3.6$\times$})}}
& \makecell{{\footnotesize \textcolor{blue}{195.5}
  (\textcolor{orange}{3.6$\times$})}}
& \makecell{{\footnotesize \textcolor{blue}{199.8}
  (\textcolor{orange}{3.6$\times$})}}
\\

\bottomrule
\end{tabular}%
}
\end{table}

\noindent\textbf{Analysis.}
Table~\ref{tab:ablation-bs-speed} evaluates the throughput scaling of the three decoding strategies as the batch size increases. Throughput improves monotonically for all methods, although the gains gradually diminish at larger batch sizes. From $B=1$ to $B=32$, greedy decoding increases from $18.5$ to $55.0$ tokens/s, corresponding to a $2.97\times$ increase due to batch scaling. Confidence-aware decoding rises from $51.7$ to $139.5$ tokens/s, or $2.70\times$, while Flash-Verify increases from $56.0$ to $199.8$ tokens/s, corresponding to the largest batch-scaling improvement of $3.57\times$.

Flash-Verify achieves the highest throughput at every batch size. Relative to greedy decoding under the same batch configuration, its speedup increases from $3.0\times$ at $B=1$ to approximately $3.6\times$ for $B\geq8$. Confidence-aware decoding provides a smaller speedup of approximately $2.5\times$--$2.8\times$. The advantage of Flash-Verify over confidence-aware decoding also grows with batch size: it is only $8.3\%$ faster at $B=1$, but becomes approximately $43.2\%$ faster at $B=32$. This trend indicates that Flash-Verify makes more effective use of the additional parallel computation available at larger batch sizes.

The throughput improvements are consistent with the number of tokens decoded per step. Flash-Verify accepts $5.7$ tokens per step, compared with $2.8$ for confidence-aware decoding and $1.0$ for greedy decoding. Throughput begins to saturate beyond $B=16$: at this batch size, Flash-Verify already reaches $186.2$ tokens/s, or approximately $93.2\%$ of its throughput at $B=32$. Therefore, $B=16$ provides a favorable efficiency point, capturing most of the maximum throughput while requiring only half the batch size.

\subsection{Sample Response}
\label{app:sample}

In the following, we present several examples of actual generation results produced by our approach under different parameter settings.

\begin{table}[h]
\centering
\caption{Example LLaDA-1.5 (without cache) — A qualitative comparison of different decoding methods.}
\vspace{0.3cm}

\noindent\rule{\textwidth}{0.4pt}
\textbf{Prompt:} \textit{Amy is taking a history test. She correctly answers 80\% of the multiple-choice questions, 90\% of the true/false questions, and 60\% of the long-answer questions. The multiple-choice and true/false questions are worth 1 point each, and the long answer questions are worth 5 points each. How many points does Amy score if there are 10 multiple-choice questions, 20 true/false questions, and 5 long answer questions?}
\noindent\rule{\textwidth}{0.4pt}

\vspace{0.5cm}
\resizebox{0.99\linewidth}{!}{ 
\begin{tabular}{|>{\columncolor{lightblue}}p{5cm}|>{\columncolor{lightgreen}}p{5cm}|>{\columncolor{lightpink}}p{5cm}|}
\hline
\rowcolor{headercolor}
\textcolor{white}{\textbf{Greedy}} & \textcolor{white}{\textbf{Confident-aware}} & \textcolor{white}{\textbf{Flash-Verify}} \\
\hline
\footnotesize

First find the number of multiple-choice questions Amy answers correctly: 80\% * 10 questions = <<10*.80=8>>8 questions
Then find the number of true/false questions Amy answers correctly: 90\% * 20 questions = <<20*.90=18>>18 questions
Then find the number of long-answer questions Amy answers correctly: 60\% * 5 questions = <<50*.60=3>>3 questions
Then find the total points Amy gets from multiple-choice questions: 8 questions * 1 point/question = <<8*1=8>>8 points
Then find the total points Amy gets from true/false questions: 18 questions * 1 point/question = <<18*1=18>>18 points
Then find the total points Amy gets from long-answer questions: 3 questions * 5 points/question = <<3*5=15>>15 points
Then add up the points from each type of question to find her total score: 8 points + 18 points + 15 points = <<8+18+15=41>>41 points

$\boxed{41}$

\vspace{0.2cm}
\textcolor{red}{\textbf{Steps: 283, Tokens: 283, Time:64.7s}}
&
\footnotesize

First find the number of multiple-choice questions Amy answers correctly: 80\% * 10 questions = <<80*.10=8>>8 questions
Then find the number of true/false questions Amy answers correctly: 90\% * 20 questions = <<90*.20=18>>18 questions
Then find the number of long-answer questions Amy answers correctly: 60\% * 5 questions = <<60*.05=3>>3 questions
Then find the total points Amy gets from multiple-choice questions: 8 questions * 1 point/question = <<8*1=8>>8 points
Then find the total points Amy gets from true/false questions: 18 questions * 1 point/question = <<18*1=18>>18 points
Then find the total points Amy gets from long-answer questions: 3 questions * 5 points/question = <<3*5=15>>15 points
Then add up the points from each type of question to find her total score: 8 points + 18 points + 15 points = <<8+18+15=41>>41 points

$\boxed{41}$

\vspace{0.2cm}
\textcolor{red}{\textbf{Steps: 66, Tokens: 277, Time: 17.2s}}
&
\footnotesize

First find the number of multiple-choice questions Amy answers correctly: 10 questions * 80\% = <<10*0.8=8>>8 questions
Then find the number of true/false questions Amy answers correctly: 20 questions * 90\% = <<20*0.9=18>>18 questions
Then find the number of long-answer questions Amy answers correctly: 5 questions * 60\% = <<5*0.6=3>>3 questions
Then find the total points from the multiple-choice questions: 8 questions * 1 point/question = <<8*1=8>>8 points
Then find the total points from the true/false questions: 18 questions * 1 point/question = <<18*1=18>>18 points
Then find the total points from the long-answer questions: 3 questions * 5 points/question = <<3*5=15>>15 points
Then add up the points from each type of question to find the total score: 8 points + 18 points + 15 points = <<8+18+15=41>>41 points

$\boxed{41}$

\vspace{0.2cm}
\textcolor{red}{\textbf{Steps: 30, Tokens: 276, Time: 13.1s}}
\\
\hline
\end{tabular}}
\end{table}

\begin{table}[h]
\centering
\caption{Example Flash-dLLM: A qualitative comparison of different decoding methods, LLaDA-1.5}
\vspace{0.3cm}

\noindent\rule{\textwidth}{0.4pt}
\textbf{Prompt:} \textit{Rong has been saving 20 coins in his piggy bank every month. Neil has been saving 2/5 times more coins in his piggy bank per month than Rong. How many coins are they having ten years after they started their savings?}
\noindent\rule{\textwidth}{0.4pt}

\vspace{0.5cm}
\resizebox{0.99\linewidth}{!}{ 
\begin{tabular}{|>{\columncolor{lightblue}}p{5cm}|>{\columncolor{lightgreen}}p{5cm}|>{\columncolor{lightpink}}p{5cm}|}
\hline
\rowcolor{headercolor}
\textcolor{white}{\textbf{Greedy}} & \textcolor{white}{\textbf{Confident-aware}} & \textcolor{white}{\textbf{Flash-Verify}} \\
\hline
\footnotesize

Rong saves 20 coins per month, so in one year, he saves 20 * 12 = <<20*12=240>>240 coins.
Neil saves 2/5 times more coins than Rong, so he saves 20 + (2/5) * 20 = 20 + 8 = 28 coins per month.
In one year, Neil saves 28 * 12 = <<28*12=336>>336 coins.
In ten years, Neil saves 336 * 10 = <<336*10=3360>>3360 coins.
Together, Rong and Neil have 2400 + 3360 = <<2400+3360=5760>>5760 coins.

$\boxed{5760}$

\vspace{0.2cm}
\textcolor{red}{\textbf{Steps: 224, Tokens: 224, Time:15.2s}}
&
\footnotesize

Rong saves 20 coins per month, so in one year, he saves 20 * 12 = <<20*12=240>>240 coins.
Neil saves 2/5 times more coins than Rong, so he saves 20 + (2/5) * 20 = 20 + 8 = 28 coins per month.
In one year, Neil saves 28 * 12 = <<28*12=336>>336 coins.
In ten years, Neil saves 336 * 10 = <<336*10=3360>>3360 coins.
Together, Rong and Neil have saved 240 + 3360 = <<240+3360=3600>>3600 coins.

$\boxed{3600}$

\vspace{0.2cm}
\textcolor{red}{\textbf{Steps: 73, Tokens: 222, Time: 7.8s}}
&
\footnotesize

Rong saves 20 coins per month, so in one year, he saves 20 * 12 = <<20*12=240>>240 coins.
Neil saves 2/5 times more coins than Rong, so he saves 20 + (2/5 * 20) = 20 + 8 = 28 coins per month.
In one year, Neil saves 28 * 12 = <<28*12=336>>336 coins.
In ten years, Neil saves 336 * 10 = <<336*10=3360>>3360 coins.
Together, Rong and Neil have saved 240 + 3360 = <<240+3360=3600>>3600 coins.

$\boxed{3600}$

\vspace{0.2cm}
\textcolor{red}{\textbf{Steps: 34, Tokens: 219, Time: 5.7s}}
\\
\hline
\end{tabular}}
\end{table}

\begin{table}[h]
\centering
\caption{Example Flash-Verify (without cache) — A qualitative comparison of different $\gamma$ values, LLaDA-1.5}
\vspace{0.3cm}

\noindent\rule{\textwidth}{0.4pt}
\textbf{Prompt:} \textit{A nurses’ station orders bandages in bulk packs of 50. On the first day, the nurses used 38 bandages and ordered one bulk pack of bandages. On the second day, they used ten fewer bandages. On the third day, they ordered two bulk packs of bandages and only used half a pack. They had 78 bandages left at the end of the third day. How many bandages did they start with on the first day?}
\noindent\rule{\textwidth}{0.4pt}

\vspace{0.5cm}
\resizebox{0.99\linewidth}{!}{ 
\begin{tabular}{|>{\columncolor{lightblue}}p{5cm}|>{\columncolor{lightgreen}}p{5cm}|>{\columncolor{lightpink}}p{5cm}|}
\hline
\rowcolor{headercolor}
\textcolor{white}{\textbf{$\gamma=0.8$}} & \textcolor{white}{\textbf{$\gamma=0.9$}} & \textcolor{white}{\textbf{$\gamma=0.95$}} \\
\hline
\footnotesize

 First, find the total number of bandages used on the first day: 38 bandages + 50 bandages = <<38+50=88>>88 bandages
Then find the total number of bandages used on the second day: 88 bandages - 10 bandages = <<88-10=78>>78 bandages
Then find the total number of bandages used on the third day: 2 * 50 bandages = <<2*50=100>>100 bandages
Then find the total number of bandages used on the third day: 100 bandages / 2 = <<100/2=50>>50 bandages
Then find the total number of bandages used over the three days: 88 bandages + 78 bandages + 50 bandages = <<88+78+50=216>>216 bandages
Then add the number of bandages left at the end of the third day to find the total number of bandages started: 216 bandages + 78 bandages = <<216+78=294>>294 bandages

Then 294 bandages

$\boxed{294}$

\vspace{0.2cm}
\textcolor{red}{\textbf{Steps: 38, Tokens: 298, Time:11.3s}}
&
\footnotesize

First, let's total how many bandages were used on the first day and how many were ordered:
- Used: 38 bandages
- Ordered: 1 bulk pack = 50 bandages

So, the total number of bandages at the end of the first day is:
\[ 38 + 50 = 88 \]

Next, let's determine how many bandages were used on the second day:
- Used: 38 - 10 = 28 bandages

So, the total number of bandages at the end of the second day is:
\[ 88 - 28 = 60 \]

Now, let's determine how many bandages were used on the third day:
- Ordered: 2 bulk packs = 100 bandages
- Used: 1/2 = 50 bandages

So, the total number of bandages at the end of the third day is:
\[ 60 + 100 - 50 = 110 \]

We know that they had 78 bandages left at the end of the third day, so the total number of bandages they started with on the first day is:
\[ 110 + 78 = 188 \]

Therefore, they started with 188 bandages on the first day.

$\boxed{188}$

\vspace{0.2cm}
\textcolor{red}{\textbf{Steps: 71, Tokens: 330, Time: 22.0s}}
&
\footnotesize

First, find the total number of bandages used on the first day: 38 bandages + 50 bandages = <<38+50=88>>88 bandages
Then find the total number of bandages used on the second day: 88 bandages - 10 bandages = <<88-10=78>>78 bandages
Then find the total number of bandages used on the third day: 2 * 50 bandages = <<2*50=100>>100 bandages
Then find the total number of bandages used over the three days: 88 bandages + 78 bandages + 100 bandages = <<88+78+100=266>>266 bandages
Then add the number of bandages left at the end of the third day to the total number of bandages used to find the total number of bandages at the start of the first day: 266 bandages + 78 bandages = <<266+78=344>>344 bandages

$\boxed{344}$

\vspace{0.2cm}
\textcolor{red}{\textbf{Steps: 57, Tokens: 271, Time: 19.8s}}
\\
\hline
\end{tabular}}
\end{table}



\end{document}